\documentclass[preprint,onefignum,onetabnum]{siamonline220329}
\usepackage{amsmath,amsfonts,amssymb} 
\usepackage{dsfont}
\usepackage{array}
\usepackage{stfloats}
\usepackage{subcaption}
\DeclareMathSizes{10}{10}{5}{5}

\usepackage{lipsum}
\usepackage{amsfonts}
\usepackage{graphicx}
\usepackage{epstopdf}
\usepackage{algorithmic}
\ifpdf
  \DeclareGraphicsExtensions{.eps,.pdf,.png,.jpg}
\else
  \DeclareGraphicsExtensions{.eps}
\fi

\usepackage{enumitem}
\setlist[enumerate]{leftmargin=.5in}
\setlist[itemize]{leftmargin=.5in}

\newsiamremark{remark}{Remark}
\newsiamremark{hypothesis}{Hypothesis}
\crefname{hypothesis}{Hypothesis}{Hypotheses}
\newsiamthm{claim}{Claim}

\headers{Invariant kernels on symmetric spaces}{N. Da Costa \textit{et al.}}

\title{Invariant kernels on Riemannian symmetric spaces\,: a harmonic-analytic approach}

\author{Natha\"el Da Costa\thanks{Division of Mathematical Sciences, School of Physical and Mathematical Sciences, Nanyang Technological University, 21 Nanyang Link, 637371, Singapore}
\and Cyrus Mostajeran\footnotemark[1]
\and Juan-Pablo Ortega\footnotemark[1]
\and Salem Said\thanks{Laboratoire Jean Kuntzman, Université Grenoble-Alpes, Grenoble, 38400, France}}

\usepackage{amsopn}

\begin{document}

\maketitle

% REQUIRED
\begin{abstract}
This work aims to prove that the classical Gaussian kernel, when defined on a non-Euclidean symmetric space, is never positive-definite for any choice of parameter. To achieve this goal, the paper develops new geometric and analytical arguments. These provide a rigorous characterization of the positive-definiteness of the Gaussian kernel, which is complete but for a limited number of scenarios in low dimensions that are treated by numerical computations. Chief among these results are the L$^{\!\scriptscriptstyle p}$-$\hspace{0.02cm}$Godement theorems (where $p = 1,2$), which provide verifiable necessary and sufficient conditions for a kernel defined on a symmetric space of non-compact type to be positive-definite. A celebrated theorem, sometimes called the Bochner-Godement theorem, already gives such conditions and is far more general in its scope but is especially hard to apply. Beyond the connection with the Gaussian kernel, the new results in this work lay out a blueprint for the study of invariant kernels on symmetric spaces, bringing forth specific harmonic analysis tools that suggest many future applications. 
\end{abstract}

% REQUIRED

% REQUIRED

\section{Introduction}
Positive-definite functions play a fundamental role in harmonic analysis and the theory of group representations, as well as in probability and statistical inference \cite{helgason2,kirillov,yaglom1,yaglom2}. In machine learning, these functions are of paramount importance in the context of kernel-based methods. Indeed, kernels are most often required to be positive-definite for Reproducing Kernel Hilbert Space (RKHS) methods and associated linear geometry and classification algorithms to apply~\cite{kernels1,kernels2}.  
While many positive-definite kernels are known in Euclidean spaces, this is not yet true for non-Euclidean metric spaces, such as non-trivial Riemannian manifolds. Even when such kernels are known, they are far from easy to work with (in the first place, to evaluate). For example, just one evaluation of the heat kernel on a Riemannian symmetric space may require an elaborate Monte Carlo scheme~\cite{iskander1,iskander2}. 

This being the case, one hopes to develop a tractable means of constructing closed-form positive-definite kernels on Riemannian manifolds or at least to develop an explicit criterion to verify whether a given closed-form expression yields a positive-definite kernel or not. This should apply, in particular, to the so-called distance kernels, which are of the form $k(x,y) = g(d(x,y))$, with $g$ a suitable function and $d(x,y)$ the Riemannian distance. 

The results developed in this work serve exactly this purpose. They are motivated by the study of the Gaussian kernel
\begin{equation} \label{eq:gausskernel}
  k(x,y) = \exp\left(-\lambda\hspace{0.02cm}d^{\hspace{0.03cm}\scriptscriptstyle 2}(x,y)\right) \hspace{0.5cm} \text{where $\lambda > 0$}.
\end{equation}
In particular, they aim to show that this kernel is never positive-definite (meaning it is not positive-definite for any value of $\lambda$) when defined on a non-Euclidean Riemannian symmetric space $M$. In general, when $M$ is a non-Euclidean Riemannian manifold, it is known that there always exists some $\lambda$ such that the Gaussian kernel is not positive-definite~\cite{Jayasumana2015,Feragen_2015_CVPR,feragen_open}. However, whether there exists any subset of values of the parameter $\lambda$ for which the Gaussian kernel is positive-definite is an open problem of theoretical and practical significance as noted in the literature~\cite{Minh2018}. For instance,~\cite{sra_positive_2013} shows that the Gaussian kernel may be positive-definite for some non-trivial collection of parameters in certain non-Euclidean geometries.
Recently, it has been shown that if $M$ is compact and not simply connected (\textit{e.g.} a torus or a real projective space), then the Gaussian kernel is never positive-definite~\cite{nonsimply}. While these general results rely on somewhat complicated geometric arguments, they will be considerably extended here, in the case of symmetric spaces, thanks to the introduction of new analytical tools. 

A typical result of choice in this circle of ideas is Bochner's theorem, one of the most famous theorems in harmonic analysis~\cite{folland}. It states that
a function defined on a locally compact Abelian group is positive-definite if and only if it is the inverse Fourier transform of some positive measure. For example, this theorem guarantees the existence of a spectral power measure for any wide-sense stationary signal (Wiener-Khinchin theorem).

The generalization of Bochner's theorem to symmetric spaces is due to Godement~\cite{godement}. It will be referred to as Godement's theorem in this work, although it is sometimes called the Bochner-Godement theorem elsewhere. 
This theorem is recalled in Section \ref{sec:godement}, where we explain why it can be challenging to apply in practice.
Focusing on non-compact symmetric spaces, the L$^{\!\scriptscriptstyle 2}$-$\hspace{0.02cm}$Godement theorem (Section \ref{sec:l2godement}) and the L$^{\!\scriptscriptstyle 1}$-$\hspace{0.02cm}$Godement theorem (Appendix \ref{sec:l1godement}), both state that in the presence of additional integrability assumptions, the difficulties inherent in its application are greatly alleviated: it turns out that a function is positive-definite if and only if its spherical transform is positive and integrable (the definition of the spherical transform is recalled in Section \ref{sec:l2godement}). 

Basic examples of applications of the L$^{\!\scriptscriptstyle 2}$-$\hspace{0.02cm}$Godement theorem are provided in Section \ref{sec:basic_eg}. Sections \ref{sec:c_ss} and \ref{sec:nc_ss} then return to the Gaussian kernel (\ref{eq:gausskernel}). Section \ref{sec:c_ss} proves that the Gaussian kernel is never positive-definite on a compact symmetric space, expanding on \cite{nathael} and \cite{nonsimply}, which prove the corresponding result on the circle and non-simply connected closed Riemannian manifolds, respectively.
Section \ref{sec:nc_ss} takes up the case of non-compact symmetric spaces (under the general assumptions of Sections \ref{sec:godement} and \ref{sec:l2godement}, compact or non-compact have the same meaning as  ``\,of compact or non-compact type\," --- no Euclidean spaces are allowed).

Section \ref{sec:nc_ss} provides a rigorous proof that deals with the case of non-compact symmetric spaces in all but a few limited scenarios in low dimensions, which are investigated computationally instead. First, in the case of the hyperbolic plane, a numerical evaluation shows that the positivity condition in the L$^{\!\scriptscriptstyle 2}$-$\hspace{0.02cm}$Godement theorem does not hold, for any $\lambda$ within a certain continuous range, starting near $\lambda = 0$. However (Lemma \ref{lem:embedh2}), any non-compact symmetric space contains an isometrically embedded hyperbolic plane, and this implies the Gaussian kernel fails to be positive-definite on any non-compact symmetric space, for that same range of values of $\lambda$. 
Section \ref{sec:nc_ss} also proves (with no recourse to numerical work) that the Gaussian kernel is never positive-definite when defined on a three-dimensional hyperbolic space. Again, this extends to any symmetric space that contains an isometrically embedded hyperbolic space (\textit{e.g.} the spaces of $n \times n$ real, complex, or quaternion, positive-definite matrices, with the requirement that $n \geq 4$ for the real case). 

Section \ref{app:sech} introduces an alternative to the Gaussian kernel (\ref{eq:gausskernel}) for hyperbolic spaces, giving it the name of hyperbolic secant or Herschel-Maxwell kernel. Once more, this is a distance kernel
\begin{equation} \label{eq:herschelkernel}
 k(x,y) = \left(\cosh(d(x,y)\right)^{-a} \hspace{0.5cm} \text{where $a > 0$}.
\end{equation}
This kernel is derived from first principles as a different generalization of the Gaussian kernel on hyperbolic spaces, similar to the famous Herschel-Maxwell derivation \cite{jaynes}. In contrast with the Gaussian kernel, it is shown to always be positive-definite on the hyperbolic plane and positive-definite on the hyperbolic space for integer values of $a$. 

The L$^{\!\scriptscriptstyle p}$-$\hspace{0.02cm}$Godement theorems (where $p = 1,2$) are most easily applied to distance kernels (for example, to the Gaussian or Herschel-Maxwell kernels). A detailed examination of their application to more general invariant kernels, which are not merely functions of distance, is one of the main tasks to be carried out in ongoing and future work (see in particular~\cite{complexkernels}).

The present work focuses on symmetric spaces and does not discuss other, more general classes of Riemannian manifolds. This restriction is rather essential if one wishes to pursue an approach based on harmonic analysis, but it is still a restriction (Stiefel manifolds are relevant to many applications, but are not symmetric spaces in general~\cite{vemuristiefel}). Fortunately,  symmetric spaces include various spaces of covariance matrices, in addition to Grassmannnians, compact Lie groups, and other
kinds of manifolds, which arise across brain-computer interface analysis, radar signal processing, computer vision, quantum information, \textit{etc.}~\cite{app1,app2,app3,app4}. Moreover, symmetric spaces are also quite useful for graph-embedding applications~\cite{embeddingspaper}. Therefore, the ongoing efforts directed at constructing positive-definite (as~in \cite{iskander1,iskander2,matern,complexkernels}) or positively decomposable (\cite{nathael1}) kernels on symmetric spaces should lead to the emergence of new kernel-based methods (for classification, hypothesis testing, \textit{etc.}) applicable to the above-mentioned fields.

The present work deals solely with kernels and does not make any mention of Gaussian processes (Gaussian processes are considered in detail in~\cite{iskander1,iskander2,matern}). This should come as no surprise since the present approach is purely analytical rather than probabilistic. However, this is not a limitation of the present work as positive-definite kernels and Gaussian processes are equivalent. Given a Gaussian process, its covariance is a positive-definite kernel. The reverse construction can be obtained by Kolmogorov's extension theorem \cite{oksendal_stochastic_2003}. Let us finally note that the connection between invariant Gaussian processes and harmonic analysis is quite strong,\linebreak since these processes are typically obtained by solving invariant stochastic partial differential equations~\cite{matern,salempde}.

There is an additional direction where methods related to Godement's theorem can be applied. The general setting of this theorem is that of Gelfand pairs, which goes beyond symmetric spaces. In particular, Gelfand pairs include certain discrete spaces, such as spaces of phylogenetic trees, very important in bio-informatics, linguistics, \textit{etc.}~\cite{Gpairs}. A tentative direction for future work is to use Godement's theorem to improve understanding of kernel methods for phylogenetic trees.
\section{Godement's theorem} \label{sec:godement}
Let $M$ be a set and let $k:M\times M\rightarrow \mathbb{C}$. Assume that $k$ satisfies
\begin{equation} \label{eq:kpd}
  \sum^N_{i,j=1} k(x_i,x_j)\,c_i\hspace{0.02cm}\bar{c}_j \geq 0
\end{equation}
for all $x_i \in M$, $c_i \in \mathbb{C}$ and all $N\in \mathbb N$, where the bar denotes complex conjugation. 
Then, $k$ is said to be a positive-definite kernel on $M$. Without loss of generality, it is enough to consider Hermitian kernels, that is, $k(x,y)$ is the conjugate of $k(y,x)$. In addition, assume that a group $G$ acts transitively on $M$; this action is denoted by $g\cdot x$, for $g \in G$ and $x \in M$. If $k$ satisfies 
\begin{equation} \label{eq:kinv}
k(g\cdot x,g\cdot y) = k(x,y), 
\end{equation} 
for all $g \in G$ and $x,y \in M$, then $k$ is said to be a $G$-invariant (or just an invariant) positive-definite kernel. Let $o \in M$ and denote by $H$ the stabiliser of $o$ in $G$. A function $f:M \rightarrow \mathbb{C}$ is called $H$-invariant if $f(h\cdot x) = f(x)$ for all $h \in H$ and $x \in M$. Using such a function, we can define $k_f:M \times M \rightarrow \mathbb{C}$,
\begin{equation} \label{eq:ftoK}
  k_f(x,y) = f(g^{-1}_{\hspace{0.02cm} \scriptscriptstyle 1}g^{\phantom{-1}}_{\hspace{0.02cm} \scriptscriptstyle 2}\!\!\!\!\!\cdot o) \text{ for $x = g^{\phantom{-1}}_{\hspace{0.02cm} \scriptscriptstyle 1}\!\!\!\!\!\cdot o$ and $y = g^{\phantom{-1}}_{\hspace{0.02cm} \scriptscriptstyle 2}\!\!\!\!\!\cdot o$}.
\end{equation}
Indeed, $H$-invariance of $f$ implies the right-hand side of (\ref{eq:ftoK}) depends only on $x$ and $y$ (and~not on the choice of $g^{\phantom{-1}}_{\hspace{0.02cm} \scriptscriptstyle 1}\!\!\!\!\!$ and $g^{\phantom{-1}}_{\hspace{0.02cm} \scriptscriptstyle 2}\!\!\!$). Now, $k_f$ is $G$-invariant (it satisfies (\ref{eq:kinv})). If $k_f$ happens to be positive-definite (to satisfy (\ref{eq:kpd})), then $f$ is said to be a positive-definite function. 

All invariant positive-definite kernels can be represented in this way, since (\ref{eq:kpd}) and (\ref{eq:kinv}) imply that $k = k_f$, where $f(x) = k(o,x)$ and $f$ is $H$-invariant and positive-definite. In particular, if $k = k_f$, then $k$ is positive-definite if and only if $f$ is positive-definite.

Note that the setup that we just presented implies that the set $M$ is necessarily isomorphic to the homogeneous manifold $G/H $, where $G$ is the group that acts transitively on $M$ and $H$ the isotropy group of some element $o \in M$. Conversely, given any group $G$ and an arbitrary subgroup $H$, we can cast the homogeneous space $G/H $ in this framework by setting $M=G/H $, $o=eH $, and considering the transitive left action of $G$ on $G/H $. Given these observations, in the sequel, $M$ will be a Riemannian symmetric space and $(G,H)$ a symmetric pair. $G$ is assumed to be semisimple, of either compact or non-compact type, and with finite center~\cite{helgason1}. Godement's theorem generalizes Bochner's classical theorem to this setting~\cite{godement}. Roughly speaking, it states that an $H$-invariant continuous function $f$ is positive-definite if and only if it is a positive combination of positive-definite spherical functions. In~\cite{bingham}, Godement's theorem is called the Bochner-Godement theorem. 

Recall that a spherical function $\varphi:M\rightarrow \mathbb{C}$ is an $H$-invariant function, such that $\varphi(o) = 1$ and $\varphi$ is moreover a joint eigenfunction of all the $G$-invariant differential operators on $M$~\cite{helgason2}. The set $\Phi$ of positive-definite spherical functions on $M$ is locally compact for the topology of uniform convergence on compact subsets of $M$~\cite{godement}. Godement's theorem states that an $H$-invariant continuous $f$ is positive-definite if and only if 
\begin{equation} \label{eq:godement}
 f(x) = \int_\Phi\,\varphi(x)\mu_f(d\varphi)
\end{equation}
where $\mu_f$ is a unique finite positive measure on $\Phi$.

It is complicated to use Godement's theorem to verify whether a given $f$ is positive-definite. When $M$ is non-compact, the set $\Phi$ may be challenging to determine~\cite{helgason2} (Page 484), and there is no straightforward way of determining $\mu_f$ from $f$.  The following section identifies a special case of Godement's theorem where these issues do not arise. 

\section{L$^\textbf{\!2}$-Godement theorem} \label{sec:l2godement}
The idea is to combine Godement's theorem with the so-called spherical transform. Roughly speaking, the spherical transform can be seen as a radial Fourier transform for $H$-invariant functions~\cite{terras1,terras2}. When $M = \mathbb{R}^2$, it boils down to a zero-order Hankel transform. 

Assume that $M$ is a non-compact symmetric space. Pick an Iwasawa decomposition of the Lie algebra $\mathfrak{g}$ of $G$, say $\mathfrak{g} = \mathfrak{n} + \mathfrak{a} + \mathfrak{h}$ (where $\mathfrak{n}$ is nilpotent, $\mathfrak{a}$ is Abelian, and $\mathfrak{h}$ is the Lie algebra of $H$~\cite{helgason1}). For each $g \in G$, $g = n(g)\exp(a(g))\hspace{0.03cm}h(g)$, where $n(g) \in N$, $a(g) \in \mathfrak{a}$ and $h(g) \in H$ are unique ($N$ is the Lie subgroup of $G$ with Lie algebra $\mathfrak{n}$). 

For each $t \in \mathfrak{a}^*$ (the dual of $\mathfrak{a}$ as a real vector space), there is a spherical function $\varphi^{}_t:M \rightarrow \mathbb{C}$, given by the following integral formula~\cite{helgason2},
\begin{equation} \label{eq:spherical_nc}
  \varphi^{}_t(g\cdot o) = \int_H \exp\left[(\mathrm{i}t+\rho)(a(hg))\right]dh \hspace{1cm} \text{for $g \in G$}.
\end{equation}
Here, $\rho$ is half the sum of the positive roots of $\mathfrak{g}$ with respect to $\mathfrak{a}$, and $dh$ denotes the normalized Haar measure on $H$ (a gentle introduction to these concepts is provided in~\cite{terras1,terras2}). Define the spherical transform of an $H$-invariant $f:M\rightarrow \mathbb{C}$ as $\hat{f}:\mathfrak{a}^* \rightarrow \mathbb{C}$, 
\begin{equation} \label{eq:sphericaltransform}
  \hat{f}(t) = \int_M f(x)\varphi^{}_{-t}(x)\hspace{0.02cm}\mathrm{vol}(dx)
\end{equation}
where $\mathrm{vol}$ denotes the volume measure of $M$. This is well-defined when $f$ is continuous and compactly supported, in which case it admits an inversion formula, 
\begin{equation} \label{eq:inverse_sphericaltransform}
  f(x) = c_{\scriptscriptstyle M}\int_{\mathfrak{a}^*} \hat{f}(t)\varphi^{}_t(x)\hspace{0.02cm}|c(t)|^{-2}dt
\end{equation}
where $c(t)$ is a certain function, called the Harish-Chandra function (for an overview, see~\cite{jewel}), $dt$ denotes the Lebesgue measure,
and $c_{\scriptscriptstyle M}$ is a constant that is independent of the function $f$. 

Defined as above, the spherical transform extends to the space $L^2_H(M)$ of $H$-invariant square-integrable functions (with respect to $\mathrm{vol}$) and satisfies the Plancherel formula
\begin{equation} \label{eq:plancherel}
  \int_M |f(x)|^2\hspace{0.03cm}\mathrm{vol}(dx) = c_{\scriptscriptstyle M}\int_{\mathfrak{a}^*} |\hat{f}(t)|^2\hspace{0.03cm}|c(t)|^{-2}dt
\end{equation}
for any $f \in L^2_H(M)$. Here, the following proposition will be called the L$^{\!\scriptscriptstyle 2}$-$\hspace{0.02cm}$Godement theorem. It is a special case of Godement's theorem, which holds for square-integrable functions.

\begin{theorem}[L$^{\!\scriptscriptstyle 2}$-$\hspace{0.02cm}$Godement Theorem] 
\label{prop:l2god}
  A continuous function $f \in L^2_H(M)$ is positive-definite if and only if  $\hat{f}(t) \geq 0$ for almost all $t \in \mathfrak{a}^*$ and $\hat{f}$ is integrable with respect to the measure $|c(t)|^{-2}dt$. 
\end{theorem}
The connection with Godement's theorem, from the previous Section \ref{sec:godement}, comes from the fact that all of the functions $\varphi^{}_t$ in (\ref{eq:spherical_nc}) are positive-definite~\cite{helgason2} (Page 484). The set of $\varphi^{}_t$ where $t \in \mathfrak{a}^*$ is therefore a subset of $\Phi$ (in fact, it is a proper subset). The proof of Theorem \ref{prop:l2god} is given in Section \ref{sec:prop_l2god}, where a second alternative proof is also discussed.

The following section spells out the spherical transform pair (\ref{eq:sphericaltransform})$-$(\ref{eq:inverse_sphericaltransform}) in two basic cases, where $M$ is a hyperbolic plane or a hyperbolic space of dimension equal to three, and provides two examples of applications of Theorem \ref{prop:l2god}.  

\section{Basic examples} \label{sec:basic_eg}
\subsection{Hyperbolic plane} $M$ is a two-dimensional simply connected space with constant negative curvature equal to $-1$. One may take $G = SL(2,\mathbb{R})$ and $H = SO(2)$ (in the usual notation, $SL(2,\mathbb{R}) = 2 \times 2$ real unit-determinant matrices, $SO(2) = 2 \times 2$ rotation matrices). A detailed exposition may be found in~\cite{terras1}. For any $H$-invariant function $f:M\rightarrow \mathbb{C}$ there exists $g:\mathbb{R}_+ \rightarrow \mathbb{C}$, with $f(x) = g(r)$ where $r = d(o,x)$, the distance in $M$ between $x$ and the ``\,point of origin\," $o$ (the choice of this point is arbitrary). In particular, the spherical functions (\ref{eq:spherical_nc}) are given in terms of Legendre functions~\cite{helgason2,terras1},
\begin{equation} \label{eq:conicals}
  \varphi^{}_t(x) = P_{-\frac{1}{2}+\mathrm{i}t}(\cosh(r))
\end{equation}
for $t \in \mathbb{R}$ (with regard to (\ref{eq:spherical_nc}), $\mathfrak{a}$ is one-dimensional, and therefore $\mathfrak{a} \simeq \mathfrak{a}^* \simeq \mathbb{R}$. Accordingly, $t \in \mathfrak{a}^*$ becomes $t \in \mathbb{R}$), and it should be noted that $\varphi^{}_t = \varphi^{}_{-t}$ (see~\cite{terras1} (Page 141)). 

The spherical transform pair (\ref{eq:sphericaltransform})$-$(\ref{eq:inverse_sphericaltransform}) can be expressed as follows (see~\cite{terras1} (Page 149))
\begin{align} 
\label{eq:h2_sphericaltransform} & \hat{f}(t) = 2\pi\int^\infty_0 g(r)\hspace{0.02cm}P_{-\frac{1}{2}+\mathrm{i}t}(\cosh(r))\hspace{0.02cm}\sinh(r)dr, \\[0.2cm]
\label{eq:h2_inverse_sphericaltransform} & g(r) = \frac{1}{2\pi}\int^\infty_0 \hat{f}(t)\hspace{0.02cm}P_{-\frac{1}{2}+\mathrm{i}t}(\cosh(r))\hspace{0.02cm}\alpha(t)dt,
\end{align}
where $f(x) = g(r)$ as above, and $\alpha(t) = t\tanh(\pi t)$. As pointed out in~\cite{terras1}, (\ref{eq:h2_sphericaltransform})$-$(\ref{eq:h2_inverse_sphericaltransform}) is known as a Mehler-Fock transform pair in the mathematical physics literature. Note that (\ref{eq:h2_sphericaltransform}) follows directly from (\ref{eq:sphericaltransform}) by using (\ref{eq:conicals}) and then expressing the volume integral in geodesic spherical coordinates, which leads to the appearance of $\sinh(r)$ under the integral (roughly, $\mathrm{vol}(dx) = \sinh(r)\hspace{0.02cm}dr\hspace{0.02cm}d\theta$ where $\theta$ is a polar angle, which yields the pre-factor $2\pi$). 

\subsection{Hyperbolic space} $M$ is now of dimension $3$ instead of $2$.  One may take $G = SL(2,\mathbb{C})$ and $H = SU(2)$ (in other words, $SL(2,\mathbb{C}) = 2 \times 2$ complex unit determinant matrices, and $SU(2) = $ unitary matrices in $SL(2,\mathbb{C})$). As before, any $H$-invariant function $f$ is of the form $f(x) = g(r)$. The spherical functions (\ref{eq:spherical_nc}) are given by~\cite{helgason2} (Page 432)
\begin{equation} \label{eq:sinsinh}
  \varphi^{}_t(x) = \frac{\sin(tr)}{t\sinh(r)}.
\end{equation}
Again, $\varphi^{}_t = \varphi^{}_{-t}\,$. The spherical transform pair (\ref{eq:sphericaltransform})$-$(\ref{eq:inverse_sphericaltransform}) is essentially a sine transform pair,
\begin{align} 
\label{eq:h3_sphericaltransform} & \hat{f}(t) = \frac{4\pi}{t}\int^\infty_0 g(r)\sinh(r)\hspace{0.02cm}\sin(tr)dr, \\[0.2cm]
\label{eq:h3_inverse_sphericaltransform} & g(r)\sinh(r) = \frac{1}{2\pi^2}\int^\infty_0 t\hat{f}(t)\hspace{0.02cm}\sin(tr)\hspace{0.02cm}dt,
\end{align}
as can be seen using (\ref{eq:sinsinh}) and computing the integral (\ref{eq:sphericaltransform}) in geodesic spherical coordinates. 

\subsection{Distance kernels} For the hyperbolic plane and space that we studied in the previous examples, any $H$-invariant function $f$ is of the form $f(x) = g(d(o,x))$. It follows from (\ref{eq:ftoK}) that $k_f(x,y) = g(d(x,y))$. But, as stated in Section \ref{sec:godement}, any invariant kernel $k$ is of the form $k = k_f\hspace{0.03cm}$ and it is hence a distance kernel
\begin{equation} \label{eq:distK}
  k(x,y) = g(d(x,y)),
\end{equation}
for some $g:\mathbb{R}_+ \rightarrow \mathbb{C}$. The positive-definiteness of $k$ implies that $g$ is necessarily real-valued. In other words, any positive-definite kernel is a real-valued distance kernel. This statement holds whenever $M$ is a rank-one symmetric space. This class of symmetric spaces includes spheres, hyperbolic spaces, projective and hyperbolic projective spaces of any dimension~\cite{helgason1}.

Intuitively, the rank of a symmetric space $M$ is the (minimum) number of geometric invariants required to determine a pair of points $(x,y)$ up to isometry. For a rank-one space, there is exactly one such invariant, the distance $d(x,y)$. Thus, if $k$ is an invariant kernel, $k(x,y)$ depends only on $d(x,y)$. 

\subsubsection{Wishart kernel} Let $M$ be the hyperbolic plane, and let $g(r) = \exp(-2a\cosh(r))$, where $a > 0$, in (\ref{eq:distK}). It is well known that $M$ can be identified with the space of $2\times 2$ unit-determinant positive-definite matrices~\cite{terras1}. Thinking of $x \in M$ as a matrix of this kind, and choosing $o$ the identity matrix,
$d(o,x) = r$ implies the trace of $x$ is $\mathrm{tr}(x) = 2\cosh(r)$~\cite{terras1} (Page 149. Exercice 20). Thus, by (\ref{eq:distK})
$$
k(o,x) = \exp\left(-a\hspace{0.02cm}\mathrm{tr}(x)\right).
$$
Recalling the action $g\cdot x = g\hspace{0.02cm}x\hspace{0.02cm}g^{\dagger}$ (a matrix product, with $\dagger$ the transpose~\cite{terras1}) and $o =$ identity,\hfill\linebreak and setting $x = g^{\phantom{-1}}_{\hspace{0.02cm} \scriptscriptstyle 1}\!\!\!\!\!\cdot o$ and $y = g^{\phantom{-1}}_{\hspace{0.02cm} \scriptscriptstyle 2}\!\!\!\!\!\cdot o$, it follows from (\ref{eq:kinv}) that
$$
k(x,y) = k(o,g^{-1}_{\hspace{0.02cm} \scriptscriptstyle 1}g^{\phantom{-1}}_{\hspace{0.02cm} \scriptscriptstyle 2}\!\!\!\!\!\cdot o) = \exp(-a\hspace{0.02cm}\mathrm{tr}(x^{-1}y)).
$$ 
This is the unnormalized density of a Wishart distribution, and $k$ is the Wishart kernel. Now, the L$^{\!\scriptscriptstyle 2}$-$\hspace{0.02cm}$Godement theorem can be used to show that this kernel is not positive-definite, for any value of $a$. The spherical transform $\hat{f}(t)$ is given in~\cite{terras1} (Page 151),
\begin{equation} \label{eq:wishartsphere}
  g(r) = \exp(-2a\cosh(r)) \text{ in (\ref{eq:h2_sphericaltransform})}\,\Longrightarrow\, \hat{f}(t) = \left(\frac{4\pi}{a}\right)^{\!\!\frac{1}{2}}K_{\mathrm{i}t}(2a)
\end{equation}
where $K_{\mathrm{i}t}$ is a modified Bessel function of the second kind with imaginary order. For fixed $a$, $K_{\mathrm{i}t}(2a)$ has infinitely many simple zeros in $t$~\cite{paris}. In particular, $\hat{f}(t)$ does not satisfy the positivity condition in the L$^{\!\scriptscriptstyle 2}$-$\hspace{0.02cm}$Godement theorem (Theorem \ref{prop:l2god}), so $k$ is not positive-definite (the theorem can be applied because $f(x) = k(o,x)$ is continuous and in $L^2_H(M)$ for any $a$).

Furthermore, since the Wishart kernel is never positive-definite on the hyperbolic plane, and since a hyperbolic space contains infinitely many isometric copies of the hyperbolic plane, the Wishart kernel is never positive-definite on the hyperbolic space. This statement can be made rigorous using the embedding lemma below (Lemma \ref{lem:embedding}).

\subsubsection{Hyperbolic secant kernel} \label{subsubsec:hsecant} Let $M$ be the hyperbolic plane and $g(r) = (\cosh(r))^{-a}$ where $a > 1/2$. Then, $f(x) = g(d(o,x))$ is clearly continuous, and $f \in L^2_H(M)$, because
$$
\int_M|f(x)|^2\hspace{0.02cm}\mathrm{vol}(dx) = 2\pi\int^\infty_0 (\cosh(r))^{-2a}\hspace{0.02cm}\sinh(r)dr = 2\pi\int^\infty_1\,u^{-2a}du \,< \infty
$$
as follows by integrating in geodesic spherical coordinates. From~\cite{terras1} (Page 151)
\begin{equation} \label{eq:sechsphere}
  g(r) = (\cosh(r))^{-a} \,\text{ in (\ref{eq:h2_sphericaltransform})}\,\Longrightarrow\, \hat{f}(t) = 2^{a-1}\frac{\sqrt{\pi}}{\Gamma(a)}\left| \Gamma((a-1/2+\mathrm{i}t)/2)\right|^2
\end{equation}
which is positive for all $t$ ($\Gamma$ is the Gamma function). Now, by the L$^{\!\scriptscriptstyle 2}$-$\hspace{0.02cm}$Godement theorem, the hyperbolic secant kernel
$$
k(x,y) = \left(\cosh(d(x,y))\right)^{-a}
$$
is positive-definite on the hyperbolic plane for any $a > 1/2$. This example of a positive-definite kernel on the hyperbolic plane does not seem to be known in the literature. We believe it is reasonable to name it the hyperbolic secant kernel. 

The case where $M$ is the hyperbolic space, rather than the hyperbolic plane, is investigated in Appendix \ref{app: hyperbolic secant 3D}.
Positive-definiteness seems to be, once more, the general rule. 

\section{Compact symmetric spaces} \label{sec:c_ss}
This section considers the case in which $M$ is a compact symmetric space (for example, a~sphere, a projective space, or a Grassmannian). The following proposition will now be proved.
\begin{proposition} 
\label{prop:compact_gausskernel}
   Let $M$ be a compact symmetric space. The Gaussian kernel (\ref{eq:gausskernel}) is not positive-definite on $M$, for any value of $\lambda$.
\end{proposition}
In other words, the Gaussian kernel is never positive-definite on a compact symmetric space. 
The case where $M$ is not simply connected (real projective space or real Grassmannian, \textit{etc.})
is already taken care of in~\cite{nonsimply}. The case where $M$ is simply connected will now be proved.
First, we prove the following result, referred to as the Embedding Lemma.
\begin{lemma}[Embedding Lemma] 
\label{lem:embedding}
  Assume $M_o$ and $M$ are metric spaces, with $M_o$ isometrically embedded in $M$. Let $k$ be a distance kernel (\textit{i.e.} a kernel of the form (\ref{eq:distK})). If $k$ is positive-definite on $M$, then it is positive-definite on $M_o$.
\end{lemma}
Here, isometric embedding means that the distance function in $M_o$ is the restriction of the distance function in $M$ (this is the isometric embedding of metric spaces, a stronger property than the isometric embedding of Riemannian manifolds). Due to (\ref{eq:distK}), this implies the kernel $k$ on $M_o$ is just the restriction of the same kernel $k$ from $M$. The fact that positive-definiteness is hereditary from $M$ to $M_o $ is immediate from (\ref{eq:kpd}). 

A corollary of the embedding lemma is that the Gaussian kernel (\ref{eq:gausskernel}) cannot be positive-definite on $M$ if it is not positive-definite on $M_o\hspace{0.03cm}$. 

\begin{lemma} \label{lem:circle}
    Any compact simply connected symmetric space $M$ contains an isometrically embedded circle $\sigma$ (the embedding being an isometric embedding of metric spaces).
\end{lemma}

\begin{proof}
Assume that $M$ is an irreducible symmetric space~\cite[Chapter VIII]{helgason1}. Denote the maximum of the sectional curvatures of $M$ by $\kappa^2$.  Then, $M$ admits a simple periodic geodesic $\sigma$ of length $2\pi/\kappa$~\cite{helgason1} (Page 334). Let $x$ be a point on $\sigma$ and choose a parameterisation $\sigma:[0,2\pi] \rightarrow M$ with $\sigma(0) = x$. Then, let $\sigma(\pi) = x^\prime$. 

It can be shown that $\sigma$ is length-minimising from $x$ to $x^\prime$ (when restricted to $[0,\pi]$) and also from $x^\prime$ to $x$ (when restricted to $[\pi,2\pi]$). Recall that the cut-locus of $x$ in $M$ is identical to its first conjugate locus~\cite{crittenden}. By definition of $\kappa$, $x$ has no conjugate points along $\sigma$ before $\sigma(\pi) = x^\prime$~\cite[Morse-Sch\"onberg Theorem, Page 86]{chavel}. Thus $\sigma$ is length-minimising from $x$ to $x^\prime$. An identical argument shows $\sigma$ is also length-minimising from $x^\prime$ to $x$.

Denote $\sigma = \sigma([0,2\pi])$ (the image of the geodesic curve $\sigma$ in $M$). This is the isometrically embedded circle announced in Lemma \ref{lem:circle}. To see this, choose any $y \in \sigma$, $y = \sigma(t)$. If $t \in [0,\pi]$ the distance in $M$ between $x$ and $y$ is $d(x,y) = t$. On the other hand, if $t \in [\pi,2\pi]$, $d(x,y) = 2\pi - t$ (note both of these claims result from the length-minimising property of $\sigma$, just obtained). In other words, the restriction of the distance function from $M$ to $\sigma$ is the same as the distance function of a circle (identified with the interval $[0,2\pi]$ with glued endpoints). 

The assumption that $M$ is irreducible can now be removed since any simply connected symmetric space is a direct product of irreducible symmetric spaces~\cite[Page 381]{helgason1}.
\end{proof}
%\hfill$\blacksquare$ \\[0.1cm]

We recall now (see~\cite{nathael}) that the Gaussian kernel is never positive-definite on the circle. With this in mind, Lemmas \ref{lem:embedding} and \ref{lem:circle} yield Proposition \ref{prop:compact_gausskernel}.

Even though it should be possible to prove Proposition \ref{prop:compact_gausskernel} using an analytic approach via Godement's theorem, the geometric approach used above is more elementary and easier to visualise.

%scope of [11]
% Godement proof

%note that $M$ contains a closed periodic geodesic of length $\frac{2\pi}{\kappa}$ where $\kappa^2$ is the maximum sectional curvature of $M$~\cite{helgason1} ()

%It remains to prove Lemma \ref{lem:circle}. %Let $\mathfrak{g}= \mathfrak{h} + \mathfrak{p}$ be the Cartan decomposition of the Lie algebra $\mathfrak{g}$ of $G$, where $\mathfrak{h}$ is the Lie algebra of $H$. Let $\mathfrak{a}$ be a maximal Abelian subspace of $\mathfrak{p}$. According to~\cite{helgason1} (Chapter V, Section 6), the set $A$ of all $\exp(a)\cdot o$ where $a \in \mathfrak{a}$ ($\exp$ denotes the Lie group exponential) is a flat, totally geodesic submanifold of $M$, and has maximal dimension for this property ($\dim(A)$ is known as the rank of $M$). Moreover (\cite{helgason1}, Chapter V, Theorem 6.2), any geodesic $\gamma$ of $M$ through $o$ is isometric to a geodesic $\gamma^\prime$ of $A$ through $o$, of~the form $\gamma^\prime(t) = h\cdot \gamma(t)$ for some $h \in H$. 

%This implies that $A$ is isometrically embedded in $M$ (an isometric embedding of metric spaces). For any $x \in A$, by the Hopf-Rinow theorem~\cite{chavel}, $M$ admits a length-minimising geodesic $\gamma$ between $o$ and $x$. The distance between $o$ and $x$ in $M$ is then the length of $\gamma$. However, $\gamma$ is isometric to a geodesic $\gamma^\prime$

\section{Non-compact symmetric spaces} \label{sec:nc_ss}
In this section, $M$  will be a non-compact symmetric space, as in Section \ref{sec:l2godement} (for example, a hyperbolic space, a space of positive-definite matrices, \textit{etc.}). Using a mixture of numerical and analytical results, we show that the Gaussian kernel systematically fails to be positive-definite on $M$. First, the L$^{\!\scriptscriptstyle 2}$-$\hspace{0.02cm}$Godement theorem is applied to the two basic cases where $M$ is a hyperbolic plane or a three-dimensional hyperbolic space. Specifically, from (\ref{eq:gausskernel}), the Gaussian kernel corresponds to $g(r) = e^{-\lambda r^2}$, and upon replacing this into (\ref{eq:h2_sphericaltransform}) and (\ref{eq:h3_sphericaltransform}), the following observations are made:
\begin{itemize}
   \item In both cases, the spherical transform $\hat{f}(t)$ is positive when $t$ is small in comparison to $\lambda$ (Proposition \ref{prop:h2_largelambda}, and the remark after Formula (\ref{eq:h3_gaussphere})).
   \item In (\ref{eq:h2_sphericaltransform}), a numerical evaluation shows that $\hat{f}(t)$ oscillates between positive and negative values,
for any $\lambda$ within a certain continuous range, starting near $\lambda = 0$ (Figure \ref{fig:h2}).
%is also positive when $t$ is large in comparison to $\lambda$ (see Proposition \ref{prop:h2_larget}). 
   \item In (\ref{eq:h3_sphericaltransform}), $\hat{f}(t)$ has the same sign as $\sin(t/2\lambda)$. The Gaussian kernel is, therefore, never positive-definite on the hyperbolic space (Proposition \ref{prop:h3_gausskernel}). 
\end{itemize}
Second, the general case uses the Embedding Lemma \ref{lem:embedding}. 

\subsection{Hyperbolic plane} \label{sec:gauss_h2} For the Gaussian kernel on the hyperbolic plane, the behavior of the spherical transform $\hat{f}(t)$, when $t$ is small in comparison to $\lambda$, is the following.  
\begin{proposition} \label{prop:h2_largelambda}
  Let $g(r) = e^{-\lambda r^2}$ in (\ref{eq:h2_sphericaltransform}). For any $T > 0$, there exists $\lambda_T > 0$ such that $\lambda > \lambda_T$ implies $\hat{f}(t) \geq 0$ for all $t \leq T$.  
\end{proposition}
The proof of Proposition \ref{prop:h2_largelambda} is given in Section \ref{sec:proof_largelambda}. The behavior of $\hat{f}(t)$ for larger values of $t$ was investigated numerically, with the results reported in Figure \ref{fig:h2}. Note that, at the time of writing, we were not able to provide a full proof or counterexample to the negativity of the function $\hat f(t_\lambda)$ at some $t_\lambda$ for all $\lambda$.

Figure \ref{fig:h2_1} shows the sign of $\hat{f}$ in the $(t,\lambda)$-plane. The blue (positive) area on the left confirms Proposition \ref{prop:h2_largelambda}. 
It is followed to the right by an alternation of red and blue areas, which shows that $\hat{f}(t)$ oscillates between positive and negative values (note that the $\lambda$ axis begins at $\lambda = 0.2$, since the numerical evaluation of the integral (\ref{eq:h2_sphericaltransform}) becomes unstable for much smaller values of $\lambda$). Figure \ref{fig:h2_2} shows an individual plot of $\hat{f}(t)$, obtained for a fixed $\lambda = 0.05$. 

It is clear from Figure \ref{fig:h2_1} that the L$^{\!\scriptscriptstyle 2}$-$\hspace{0.02cm}$Godement theorem implies the Gaussian kernel is not positive-definite on the hyperbolic plane, for any $\lambda$ in the range of the $\lambda$-axis of this figure.  
\begin{remark}
Proposition \ref{prop:h2_largelambda} continues to hold for a general non-compact $M$, as above. The proof of this general statement is not provided in order to maintain a reasonable length.
\end{remark}
%Before going on, it should be noted that Proposition \ref{prop:h2_largelambda} continues to hold for a general non-compact $M$, as above. Here, the proof of this general statement is not provided, in order to maintain a reasonable length. 
\begin{figure}
\centering
\begin{subfigure}{.55\textwidth}
  \centering
\includegraphics[width=1.0\linewidth]{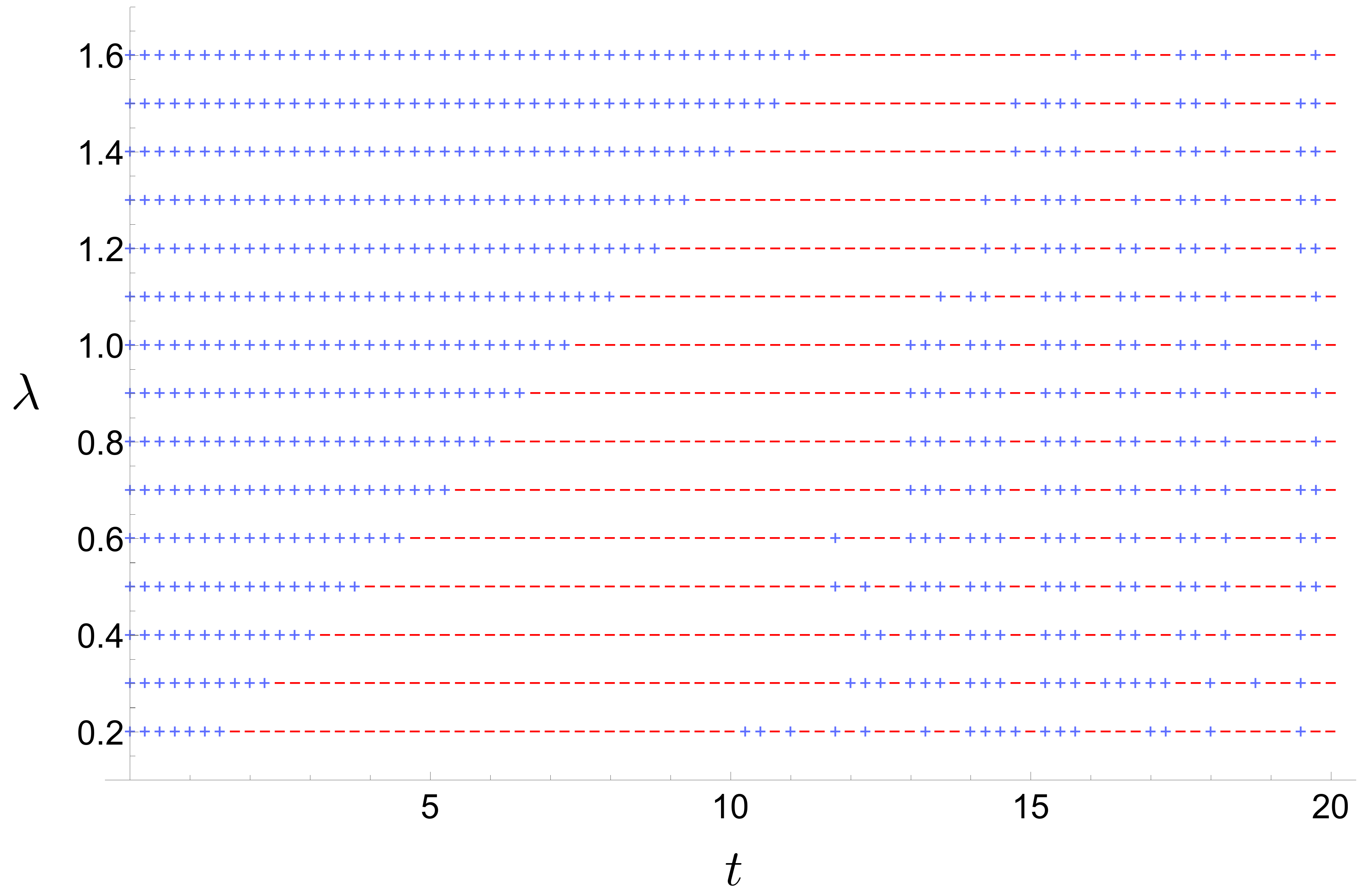}
    \caption{{Sign of $\hat{f}$ at a grid of points in the $(t,\lambda)$-plane}}
\label{fig:h2_1}
\end{subfigure}%
\begin{subfigure}{.44\textwidth}
  \centering
    \includegraphics[width=1.0\linewidth]{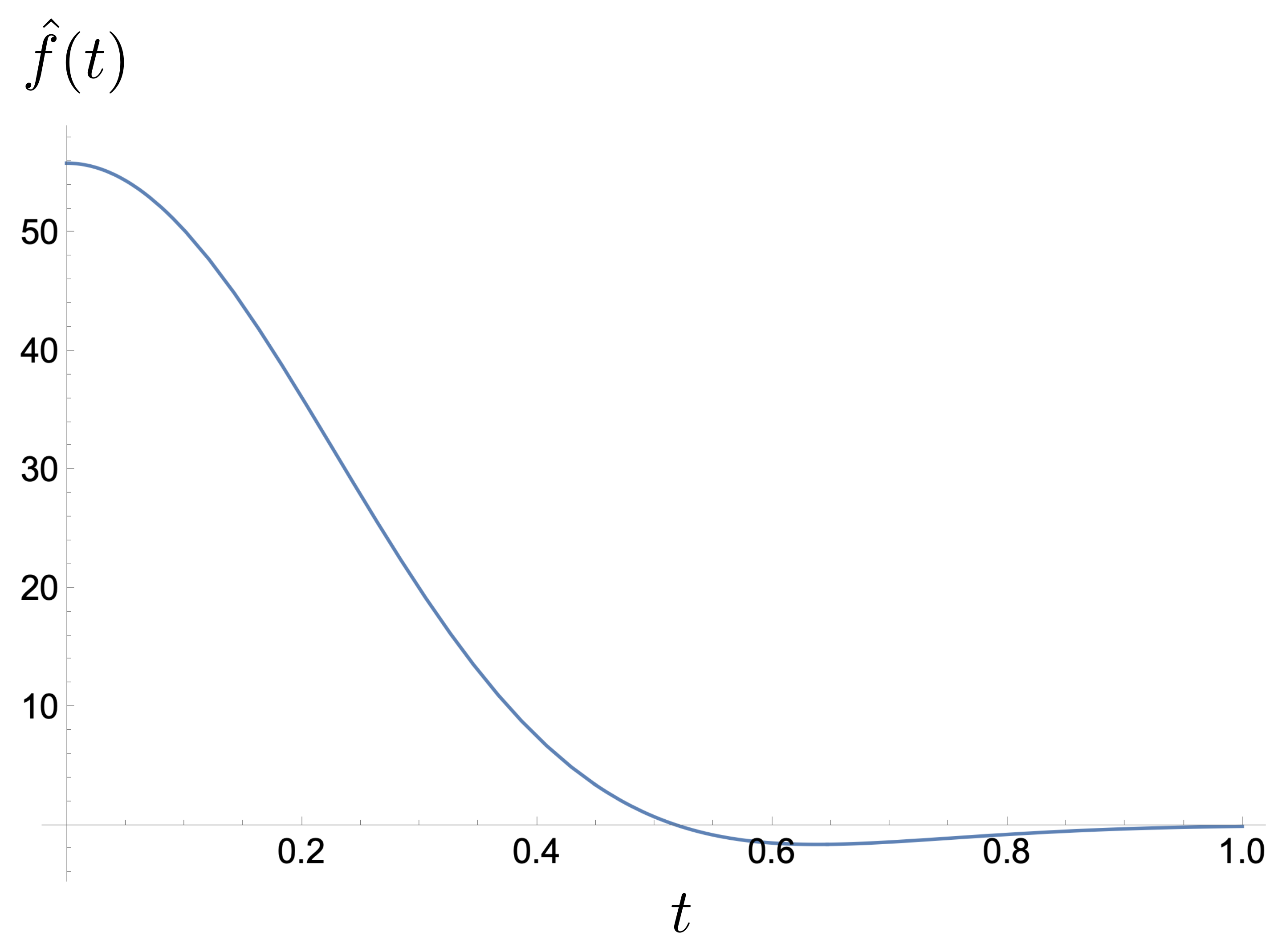}
    \caption{{Plot of $\hat{f}(t)$ when $\lambda = 0.05$}}
\label{fig:h2_2}
\end{subfigure}
\caption{Spherical transform for the Gaussian kernel on the hyperbolic plane}
\label{fig:h2}
\end{figure}

%\newpage

%\begin{proposition} \label{prop:h2_larget}
% let $g(r) = e^{-\lambda r^2}$ in (\ref{eq:h2_sphericaltransform}). Then, $\hat{f}(t) > 0$ when $t$ is sufficiently large. Precisely, $\lambda$ being fixed,
%$
%   \hat{f}(t) \sim \frac{\pi}{\sqrt{2}}\,t^{-3}
%$
%as $t \rightarrow \infty$ ($a \sim b$ means the ratio $a/b$ converges to $1$). 
%\end{proposition}
%The behavior described in Proposition \ref{prop:h2_largelambda} is quite general. In fact, Proposition \ref{prop:h2_largelambda} holds true for any non-compact symmetric space, and the proof of this general statement is very similar to the one given in Section \ref{sec:proofs_gauss_h2} (this is not pursued, in order to maintain a reasonable length).  
%
%Proposition \ref{prop:h2_largelambda} and \ref{prop:h2_larget} do not say anything about negative values of $\hat{f}(t)$. However, $\ldots$
%
\subsection{Hyperbolic space} \label{sec:gauss_h3}
In this case, $\hat{f}(t)$ can be found in closed form.
\begin{equation} \label{eq:h3_gaussphere}
  g(r) =  e^{-\lambda r^2} \text{ in (\ref{eq:h3_sphericaltransform})}\,\Longrightarrow\, \hat{f}(t) = \left(\frac{2\pi}{t}\right)\!\left(\frac{\pi}{\lambda}\right)^{\!\frac{1}{2}}\hspace{0.03cm} e^{(1-t^2)/4\lambda} \sin\!\left(\frac{t}{2\lambda}\right).
\end{equation}
This shows that the sign of $\hat{f}(t)$ is determined by $\sin(t/2\lambda)$. This remains positive for $t \leq 2\pi\lambda$, and then oscillates indefinitely. From the L$^{\!\scriptscriptstyle 2}$-$\hspace{0.02cm}$Godement theorem, the following is immediate.
\begin{proposition} \label{prop:h3_gausskernel}
     The Gaussian kernel (\ref{eq:gausskernel}) is not positive-definite on the hyperbolic space for any value of $\lambda$. 
\end{proposition}
%The L$^{\!\scriptscriptstyle 2}$-$\hspace{0.02cm}$Godement theorem is indeed applicable to this kernel, since the corresponding function $f(x) = k(o,x)$ is continuous and in $L^2_H(M)$, for any $\lambda > 0$. 
%
The expression (\ref{eq:h3_gaussphere}) can be obtained from the following general lemma, proved in Section \ref{sec:proof_h3_sine}.
\begin{lemma} \label{lem:h3_sine}
In (\ref{eq:h3_sphericaltransform}),  extend $g(r)$ to an even function of $r \in \mathbb{R}$. Then,
\begin{equation} \label{eq:h3_sine}
  \hat{f}(t) = \frac{\pi}{t}\,\mathrm{Im}\left\lbrace G(\mathrm{i}t+1) - G(\mathrm{i}t-1)\right\rbrace
\end{equation}
where $\mathrm{Im}$ denotes the imaginary part and $G$ is the moment-generating function
$$
G(s) =\int^{+\infty}_{-\infty} g(r)\hspace{0.02cm}e^{sr}dr.
$$
\end{lemma}
Indeed, (\ref{eq:h3_gaussphere}) follows from (\ref{eq:h3_sine}) using the expression of the moment-generating function of a univariate normal distribution, 
$$
G(s) = \left(\frac{\pi}{\lambda}\right)^{\!\frac{1}{2}}e^{s^2/4\lambda}.
$$
%when $G(s)$ is the moment-generating function of a univariate normal distribution. %In the notation of (\ref{eq:h3_sine}), 

%remark about L2    

%$$
%G(s) = \left(\frac{\pi}{\lambda}\right)^{\!\frac{1}{2}}e^{s^2/4\lambda}
%$$
%so that (\ref{eq:h3_gaussphere}) becomes immediate.     
%
\subsection{The general case}
Let $M$ be any non-compact symmetric space, as in Section \ref{sec:l2godement}\hfill\linebreak (so the isometry group $G$ of $M$ is semisimple, of non-compact type, and has finite center~\cite{helgason1}).  The following lemma will be proved in Section \ref{sec:proof_embedh2}.
\begin{lemma} \label{lem:embedh2}
  Under the assumptions just stated, $M$ contains an isometrically embedded\hfill\linebreak hyperbolic plane.
\end{lemma}
The numerical results of Paragraph \ref{sec:gauss_h2} show that the Gaussian kernel is not positive-definite on the hyperbolic plane, at least for any $\lambda$ in the range appearing in Figure \ref{fig:h2_1}. The Embedding Lemma \ref{lem:embedding} and Lemma \ref{lem:embedh2} therefore imply that the Gaussian kernel is not positive-definite on any non-compact Riemannian symmetric space, for $\lambda$ in this same range
(the additional assumptions of Lemma \ref{lem:embedh2} merely serve to prevent $M$ from being Euclidean).  

On the other hand, Proposition \ref{prop:h3_gausskernel} shows the Gaussian kernel is never positive-definite on the hyperbolic space. By the Embedding Lemma, it is therefore never positive-definite on symmetric spaces that contain an isometrically embedded hyperbolic space. Examples of such symmetric spaces include all spaces of $n \times n$ real, complex, or quaternion, positive-definite matrices (with the requirement that $n \geq 4$ for the real case)~\cite{embeddingspaper}. 

Hopefully, the numerical treatment of the hyperbolic plane case in Paragraph \ref{sec:gauss_h2} will soon be replaced with a complete analytical treatment, similar to the one that we conducted for the hyperbolic space case in Paragraph \ref{sec:gauss_h3}. This is the only step still missing to obtain a fully analytical proof of the statement that the Gaussian kernel is never positive-definite on any non-Euclidean symmetric space. 

%From the embedding lemma, and Propositions \ref{prop:h3_gausskernel} and \ref{prop:embedh3}, the following is clear.
%\begin{proposition} \label{prop:general_nc}
%  Let $M$ be one of the classical non-compact Riemannian symmetric spaces. If $\dim M \geq 3$, then the Gaussian kernel (\ref{eq:gausskernel}) is not positive-definite on $M$, for any value of~$\lambda$. 
%\end{proposition}
%For example, the Gaussian kernel is never positive-definite on a space of $n\times n$ positive-definite matrices with real entries, if $n \geq 3$, and it is never positive-definite on a space of $n \times n$ positive-definite matrices with complex or quaternion entries (irrespective of the matrix dimension $n$). 

%Proposition \ref{prop:general_nc} covers all the classical non-compact symmetric spaces, except for the hyperbolic plane. The numerical results of Paragraph \ref{sec:gauss_h2} seem to suggest this proposition will soon be extended to all the classical non-compact symmetric spaces. 

%The proof of Proposition \ref{prop:embedh3} is given in Section \ref{sec:proof_embedh3}. It is
%Paragraph \ref{sec:gauss_h2} only dealt with this case numerically. However, 

%\section{} \label{sec:alternative}

%% $M$ always includes H2

%% certain $M$ include H3

\section{The hyperbolic secant revisited: non-Euclidean Herschel-Maxwell distributions} \label{app:sech}

Recall the hyperbolic secant kernel on the hyperbolic plane from Section \ref{subsubsec:hsecant}, which was shown to be positive-definite in contrast to the Gaussian kernel. Here, we will present a new interpretation of the hyperbolic secant kernel as the Herschel-Maxwell kernel on the hyperbolic plane. In Euclidean space, the Herschel-Maxwell theorem is an elegant characterization of the Gaussian distribution originally due to the astronomer J.F.W. Herschel, who was studying two-dimensional errors in astronomical observations. Ten years later, James Clerk Maxwell presented a three-dimensional version of the same argument to show that the stationary probability distribution of the velocities of molecules in a gas follows a Gaussian distribution under the assumptions that (P1) the components of the random vector in an orthogonal coordinate system are independent, and (P2) the distribution only depends on the magnitude of the vector~\cite{Gyenis,jaynes,maxwell1860}. A modern statement of the theorem is that if the distribution of a random vector with independent components is invariant under rotations, then the components must be identically distributed as a Gaussian distribution.

Here, we will follow the Herschel-Maxwell derivation as it appears in Jaynes~\cite{jaynes}, which relies on the assumption of the existence of densities of the components of the random variables and solving a functional equation. This derivation closely follows Maxwell's own argument. Let $(u,v)$ be a Cartesian coordinate system in the Euclidean plane. We seek to determine a continuous joint probability distribution $p(u,v)dudv$ that satisfies assumptions P1 and P2. By assumption P1, we have 
\begin{equation} \label{HM independence}
   p(u,v)dudv=h(u)du \, h(v)dv 
\end{equation} for some function $h$. By assumption P2, we have 
\begin{equation} \label{HM distance}
   p(u,v)dudv=g(r)r dr d\theta
\end{equation}
as the density cannot depend on the angle $\theta$ when expressed in polar coordinates $(r,\theta)$. (\ref{HM independence}) and (\ref{HM distance}) combine to give the functional equation
\begin{equation}
    h(u)h(v)=g\left(\sqrt{u^2+v^2}\right),
\end{equation}
which admits the general solution
\begin{equation}
    h(u)=\sqrt{\frac{\lambda}{\pi}}e^{-\lambda u^2}, \quad \quad \quad p(u,v)=\frac{\lambda}{\pi}e^{-\lambda(u^2+v^2)}, 
\end{equation}
for any $\lambda > 0$
upon insisting that the distribution be normalized. 

To replicate this derivation in the hyperbolic plane, we must first specify what we mean by an `orthogonal coordinate system' in the hyperbolic plane. The most natural choice for such a coordinate system seems to be offered by Lobachevsky coordinates, determined by specifying a directed geodesic ($u$-axis) through the origin. The Lobachevsky coordinates of a point $x$ are then found by dropping a perpendicular to the $u$-axis. The signed distance from the foot of the perpendicular to the origin is the $u$-coordinate of the point, and the $v$-coordinate is the signed distance along the perpendicular to the $u$-axis (with the distance taken to be positive on one side of the $u$-axis and negative on the other). With respect to such a coordinate system, we seek a joint probability distribution $p(x)\hspace{0.02cm}\mathrm{vol}(dx) = p(u,v)\hspace{0.02cm}\mathrm{vol}(dx)= h(u)h(v)\hspace{0.02cm}\mathrm{vol}(dx)$, where $\mathrm{vol}(dx)$ denotes the Riemannian volume form in the hyperbolic plane (P1). Moreover, (P2) generalizes to the assumption that  $p(x)$ can only depend on the Riemannian distance to the origin $r=d(o,x)$, which once again yields the functional equation $h(u)h(v)=g(r)$. As this equation holds for all $(u,v)$, we can take any point on the $u$-axis to obtain $g(u)=h(u)h(0)$. That is, we have the functional equation
\begin{equation} \label{h functional equation}
    h(u)h(v)=h(0)h(r)
\end{equation}
where $(u,v)$ and $r$ are related by
\begin{equation}
    \cosh(r) = \cosh(u) \cosh(v)
\end{equation}
according to the hyperbolic Pythagorean theorem. Setting $\tilde{h}=(h\circ \mathrm{arcosh})/h(0)$, (\ref{h functional equation}) becomes
$ \tilde{h}(\cosh u)\tilde{h}(\cosh v) = \tilde{h}(\cosh u \cosh v)$, which admits the general solution $\tilde{h}(\cosh u) = (\cosh u)^{-a}$, for $a \in \mathbb{R}$ (i.e., $\tilde{h}$ is a power function). Thus,
\begin{equation*}
    h(r)=h(0)(\cosh(r))^{-a},
\end{equation*}
which yields the probability density
\begin{equation*}
   p(x)\hspace{0.02cm}\mathrm{vol}(dx) = h(0)^2 (\cosh(r))^{-a} \hspace{0.02cm} \mathrm{vol}(dx) = h(0)^2 (\cosh(r))^{-a} \sinh(r) dr d\theta,
\end{equation*}
in geodesic polar coordinates. Note that this density is normalizable only if $a > 1$. The constant $h(0)$ can be determined through normalizing the distribution, which yields 
\begin{equation} \label{2D hyperbolc HM}
    p(x)\hspace{0.02cm}\mathrm{vol}(dx) = \frac{1}{\pi}\hspace{0.02cm}\frac{\Gamma\left(\frac{a+1}{2}\right)}{\Gamma\left(\frac{a-1}{2}\right)} \left(\cosh(d(o,x)\right)^{-a} \hspace{0.02cm}\mathrm{vol}(dx).
\end{equation}
We call this distribution the Herschel-Maxwell distribution in the hyperbolic plane. It can be viewed as one of several possible generalizations of the Gaussian distribution centered at the point $o$ to the hyperbolic plane alongside the distinct distributions derived from the heat kernel interpretation of the Gaussian~\cite{heat_kernel,McKean1970} and 
\begin{equation} \label{2D hyperbolic Gauss}
     p(x)\hspace{0.02cm}\mathrm{vol}(dx) =
     \frac{1}{\pi} \sqrt{\frac{\lambda}{\pi}}\frac{e^{-\frac{1}{4\lambda}} }{\mathrm{erf}\left(\frac{1}{2\sqrt{\lambda}}\right)} e^{-\lambda \, d(o,x)^2} \hspace{0.02cm}\mathrm{vol}(dx),
     \end{equation}
which inherits an important characteristic property of Euclidean Gaussian distributions from statistical inference: the maximum likelihood estimate of the mean reduces to a (Riemannian) center of mass computation~\cite{Said2017,Said2018,HOS2022,Said2023,Chen2022}. Figure \ref{fig: Gaussian plots} compares plots of the density functions $p(r)\hspace{0.02cm}\mathrm{vol}(dx)$ in (\ref{2D hyperbolc HM}) and (\ref{2D hyperbolic Gauss}) as functions of $r=d(o,x)$ for parameters $a=4$ and $\lambda = 1.66$.

\begin{figure}
\centering
    \includegraphics[width=0.66\linewidth]{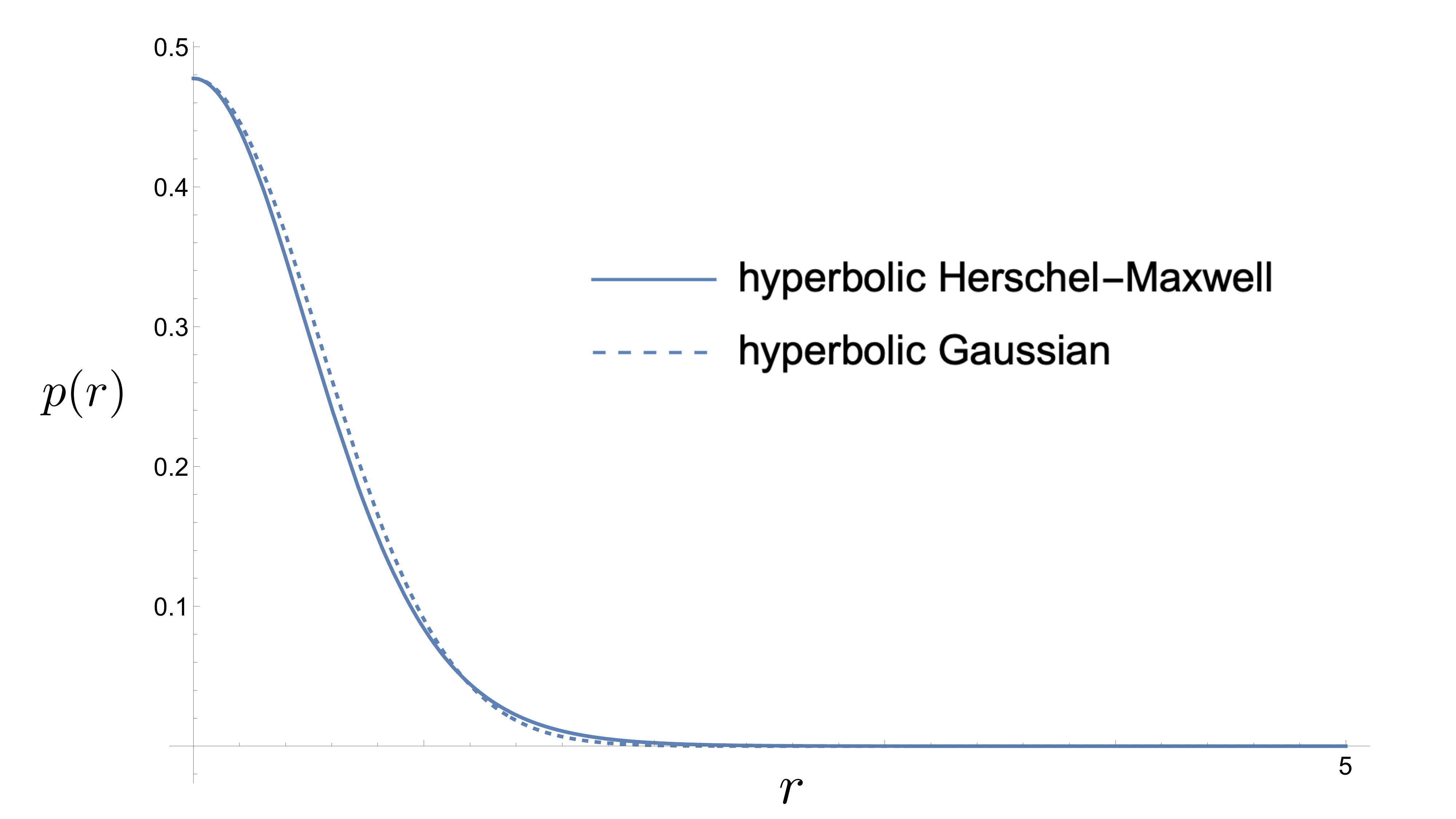}
\caption{Plots of the normalized hyperbolic Herschel-Maxwell (solid curve, (\ref{2D hyperbolc HM})) and Gaussian (dashed curve, (\ref{2D hyperbolic Gauss})) densities as functions of $r=d(o,x)$ for parameters $a=4$ and $\lambda = 1.66$.}
\label{fig: Gaussian plots}
\end{figure}

Finally, we note that the hyperbolic Herschel-Maxwell distribution can be generalized to $n$-dimensional hyperbolic space if we insist that the components of
the point $x$ in any Lobachevsky coordinate system in \emph{any} hyperbolic plane containing $o$ and $x$ are independent and the distribution only depends on the geodesic distance $r=d(o,x)$. Reasoning as before, we obtain
\begin{equation} \label{general hyperbolic HM}
   p(x)\hspace{0.02cm}\mathrm{vol}(dx) = \frac{1}{\omega_{n-1}}\hspace{0.02cm}\frac{\Gamma\left(\frac{a+1}{2}\right)}{\Gamma\left(\frac{a+1-n}{2}\right)\Gamma\left(\frac{n}{2}\right)}\left(\cosh(d(o,x)\right)^{-a} \hspace{0.02cm}\mathrm{vol}(dx),
\end{equation}
where $\omega_{n-1}$ denotes the area of the unit sphere in $\mathbb{R}^n$ and the normalization factor in (\ref{general hyperbolic HM}) is determined by integration with respect to the volume form $\mathrm{vol}(dx)$ in $n$-dimensional hyperbolic space. The distribution is well-defined for $a > n-1$.

\begin{remark}
    We can repeat the derivation of the hyperbolic Herschel-Maxwell distribution presented here in the case of spherical geometry, with the main difference being the replacement of the hyperbolic Pythagorean theorem with the spherical Pythagorean theorem
    \begin{equation}
        \cos(\theta)=\cos(u)\cos(v)
    \end{equation}
    where $o,x$ denote points on the 2-sphere $S^2$, $\theta = d(o,x)$, and $(u,v)$ is an orthogonal coordinate system centered on $o$, whereby the coordinates of a point $x$ are obtained by dropping a perpendicular to a specified geodesic through $o$ and measuring the signed lengths $u$ along the geodesic from $o$ to the foot of the perpendicular and $v$ along the perpendicular to the geodesic from $x$. This yields the spherical Herschel-Maxwell kernel
    \begin{equation} \label{spherical HM kernel}
        k(x,y)=|\cos(d(x,y))|^a,
    \end{equation}
    where $a \geq 0$ to ensure the corresponding distribution is finite. Note that by representing $x,y\in S^2$ by unit vectors $w_x, w_y\in \mathbb{R}^3$, we have $\cos(d(x,y))=\langle w_x,w_y\rangle$, and so (\ref{spherical HM kernel}) takes the form $k(x,y)=|\langle w_x,w_y\rangle|^a$. Thus, for non-negative even integer $a$, the spherical Herschel-Maxwell distribution can be identified with a homogeneous polynomial kernel $\langle x,y\rangle^n$, which is known to be positive-definite if and only if $n$ is a non-negative integer~\cite{Smola2000}.
\end{remark}
 
\section{Proofs of the main results}\label{sec:proofs}

\subsection{Proof of Theorem \ref{prop:l2god}} \label{sec:prop_l2god}

\subsubsection{Notation} \label{subsec:notation} The following notation will be used in the proof.
\begin{itemize}
\item Let $W$ denote the Weyl group, which arises from the adjoint action of $H$ on $\mathfrak{a}$~\cite{helgason1}. Also, let $\alpha$ denote the measure on $\mathfrak{a}^*$, which is given by $\alpha(dt) = c_{\scriptscriptstyle M}|c(t)|^{-2}dt$. Then, $L^2_W(\mathfrak{a}^*)$ stands for the space of $W$-invariant square-integrable functions (with respect to the measure $\alpha$).

\item Denote $\Phi_*$ the set of functions $\varphi^{}_t$ in (\ref{eq:spherical_nc}) where $t \in \mathfrak{a}^*$. This is a proper subset of the set $\Phi$ of all positive-definite spherical functions, which appears in Godement's theorem (Section~\ref{sec:godement}, in particular (\ref{eq:godement})). 

\item Functions or measures on $\mathfrak{a}^*$ can be pulled back to $\Phi_*$, assuming they are $W$-invariant. This is because $\varphi^{}_{t^\prime} = \varphi^{}_t$ if and only if $t^\prime = w\cdot t$ for some $w \in W$. Accordingly, the measure $\alpha$ on $\mathfrak{a}^*$ can be identified with a measure on $\Phi_*$\hspace{0.03cm}, also denoted $\alpha$. Similarly, if $\hat{f}(t)$ is a spherical transform, as in (\ref{eq:sphericaltransform}), $\hat{f}(\varphi)$ is just $\hat{f}(t)$ when $\varphi \in \Phi_*$ and $\varphi = \varphi^{}_t$ for some $t \in \mathfrak{a}^*$.

\item Finally, $C^{\hspace{0.03cm}c}_H(M)$ is the space of $H$-invariant continuous, compactly supported functions, with complex values, defined on $M$. 
\end{itemize}

\subsubsection{The if part} \label{subsec:ifpart} For square-integrable functions, the spherical transform is only defined in an abstract sense. It is a linear isometry between the Hilbert spaces $L^2_H(M)$ and $L^2_W(\mathfrak{a}^*)$, $f \mapsto \hat{f}$~\cite{helgason2}\cite{terras2}. However, if $\hat{f}$ is integrable, the inversion formula (\ref{eq:inverse_sphericaltransform}) holds
\begin{equation} \label{eq:inversion}
f(x) = \int_{\mathfrak{a}^*} \hat{f}(t)\varphi^{}_t(x)\hspace{0.02cm}\alpha(dt) \hspace{0.5cm} \text{for all $x \in M$}
\end{equation}
and this implies that $f$ is continuous, since each $\varphi^{}_t$ is continuous and verifies $|\varphi^{}_t(x)| \leq 1$, being a positive-definite function with $\varphi^{}_t(o) = 1$. The fact that $f$ is positive-definite follows from Godement's theorem. Indeed, (\ref{eq:inversion}) can be written
\begin{equation} \label{eq:inversiontogod}
  f(x) = \int_{\Phi_*} \varphi(x)(\hat{f}\alpha)(d\varphi)
\end{equation}
where $(\hat{f}\alpha)(d\varphi) = \hat{f}(\varphi)\alpha(d\varphi)$. Then, since $\hat{f}(\varphi) \geq 0$ for almost all $\varphi \in \Phi_*$ and $\hat{f}$ is integrable, with respect to $\alpha$,  $\hat{f}\alpha$ is a finite positive measure on $\Phi_*\hspace{0.03cm}$. Moreover, since $\Phi_*$ is a (measurable) subset of $\Phi$, (\ref{eq:inversiontogod}) is exactly of the form (\ref{eq:godement}), but with the measure $\mu_f$ supported on $\Phi_*\hspace{0.03cm}$, where it is equal to $\hat{f}\alpha$. This shows that $f$ is positive-definite, and even determines the finite positive measure $\mu_f$ (recall that $\mu_f$ is unique).

\subsubsection{The only-if part} \label{subsec:onlyif} If $f$ is $H$-invariant, continuous and positive-definite, then one has
\begin{equation} \label{eq:generalpd}
  \int_M (\psi*\psi^{\scriptscriptstyle \dagger})(x)f(x)\hspace{0.02cm}\mathrm{vol}(dx) \geq 0
\end{equation}
for any $\psi \in C^{\hspace{0.03cm}c}_H(M)$, where $*$ denotes the convolution and $\psi^{\scriptscriptstyle \dagger}(g\cdot o) = \bar{\psi}(g^{-1}\cdot o)$~\cite{dixmier} (Page 286). This yields a scalar product on $C^{\hspace{0.03cm}c}_H(M)$,
\begin{equation} \label{eq:planch1}
  \langle \psi^{\!}_{\scriptscriptstyle 1},\psi^{\!}_{\scriptscriptstyle 2}\rangle_f =    \int_M (\psi^{\,}_{\scriptscriptstyle 1}\!*\psi^{\scriptscriptstyle \dagger}_{\scriptscriptstyle 2})(x)f(x)\hspace{0.02cm}\mathrm{vol}(dx).
\end{equation}
It will shortly be proved that this scalar product admits a spectral representation,
\begin{equation} \label{eq:planch2}
\left\langle \psi_{\scriptscriptstyle 1},\psi_{\scriptscriptstyle 2}\right\rangle_f = \langle \hat{\psi}_{\scriptscriptstyle 1},\hat{\psi}_{\scriptscriptstyle 2}\rangle_{\hat{f}}
\end{equation}
where the scalar product on the right-hand side is the usual ``$L^2\hspace{0.03cm}$" scalar product, with respect to the measure $\hat{f}\alpha$ on $\Phi_*$. By~\cite{godement} (Plancherel theorem, Page 101), this implies the measure $\mu_f$ (in (\ref{eq:godement})) is supported on $\Phi_*$, where it is equal to $\hat{f}\alpha$. Since $\mu_f$ is a finite positive measure, $\hat{f}(t) \geq 0$ for almost all $t$ and $\hat{f}$ is integrable. 

To prove (\ref{eq:planch2}), consider the left-hand side of (\ref{eq:generalpd}). Both $\psi$ and $\psi^{\scriptscriptstyle \dagger}$ belong to $C^{\hspace{0.03cm}c}_H(M)$.\hfill\linebreak The same is therefore true of their convolution. Since the spherical transform of $\psi^{\scriptscriptstyle \dagger}$ is the complex conjugate of the spherical transform of $\psi$, 
the convolution property of spherical transforms (see~\cite{terras2} (Page 88)) implies that
\begin{equation} \label{eq:convtheorem}
  \widehat{\psi*\psi^{\scriptscriptstyle \dagger}}(t) = |\hat{\psi}(t)|^2.
\end{equation}
Since both $f$ and $\psi*\psi^{\scriptscriptstyle \dagger}$ are in $L^2_H(M)$, the Plancherel formula (\ref{eq:plancherel}) yields
\begin{equation} \label{eq:planch25}
\int_M (\psi*\psi^{\scriptscriptstyle \dagger})(x)f(x)\hspace{0.02cm}\mathrm{vol}(dx) = 
\int_{\mathfrak{a}^*} |\hat{\psi}(t)|^2\hat{f}(t)\hspace{0.02cm}\alpha(dt)
= \int_{\mathfrak{a}^*} |\hat{\psi}(t)|^2\hspace{0.02cm}(\hat{f}\alpha)(dt).
\end{equation}
By (\ref{eq:planch1}), and by pulling back the integral of the right-hand side to $\Phi_*\hspace{0.03cm}$, this is
\begin{equation} \label{eq:planch3}
  \left\langle \psi,\psi\right\rangle_f = \int_{\Phi_*} |\hat{\psi}(\varphi)|^2\hspace{0.03cm}(\hat{f}\alpha)(d\varphi) = 
\langle \hat{\psi},\hat{\psi}\rangle_{\hat{f}}
\end{equation}
so that (\ref{eq:planch2}) follows immediately, using the polarization identity for scalar products. 
\hfill \proofbox

\subsubsection{Alternative proof of Theorem \ref{prop:l2god}} \label{subsec:alternative} The only-if part of the proposition admits at least one more proof. Indeed~\cite{dixmier} (Page 302), for any continuous positive-definite $f \in L^2_H(M)$, it is possible to construct $\phi \in L^2_H(M)$ such that $f = \phi * \phi^{\scriptscriptstyle \dagger}$. If one can apply the convolution property as in (\ref{eq:convtheorem}), it would follow that $\hat{f}(t) = |\hat{\phi}(t)|^2$, so $\hat{f}(t)$ is positive for all $t$ (not just for almost all $t$), and $\hat{f}$ is integrable.

The convolution property does not apply directly because $\phi$ is not integrable. Letting $(\phi_n;n = 1,2,\ldots)$ belong to $C^{\hspace{0.03cm}c}_H(M)$ and converge to $\phi$ in $L^2_H(M)$, the functions $f_n = \phi_n * \tilde{\phi}_n$\hfill\linebreak belong to $C^{\hspace{0.03cm}c}_H(M)$ and converge uniformly to $f$, as can be shown by using Young's convolution inequality~\cite{hewittross} (Page 296). 

Now, as in (\ref{eq:convtheorem}), one has $\hat{f}_n(t) = |\hat{\phi}_n(t)|^2$ for each $n$. However, since $\hat{\phi}_n$ converges to $\hat{\phi}$ 
in~$L^2_W(\mathfrak{a}^*)$, it follows that $|\hat{\phi}_n|^2$ converges to $|\hat{\phi}|^2$ in $L^1_W(\mathfrak{a}^*)$ (this is the space of $W$-invariant integrable functions)~\cite{bogachev} (Page 298). 

Therefore, a representation of $f$ under the form (\ref{eq:inverse_sphericaltransform}) can be obtained by noting, for $x \in M$,
$$
\lim f_n(x) = \lim\int_{\mathfrak{a}^*}\varphi^{}_t(x)|\hat{\phi}_n(t)|^2\hspace{0.02cm}\alpha(dt) = \int_{\mathfrak{a}^*}\varphi^{}_t(x)|\hat{\phi}(t)|^2\hspace{0.02cm}\alpha(dt)
$$
where the first equality follows from (\ref{eq:inverse_sphericaltransform}) and $\hat{f}_n(t) = |\hat{\phi}_n(t)|^2$, and the second equality from
 $|\hat{\phi}_n|^2 \longrightarrow |\hat{\phi}|^2$ in $L^1_W(\mathfrak{a}^*)$ and $|\varphi^{}_t(x)| \leq 1$ (this is a consequence of $\varphi^{}_t$ being positive-definite).

However, the limit of $f_n(x)$ is just $f(x)$, since the $f_n$ converge uniformly to $f$. Thus,
$$
f(x) = \int_{\mathfrak{a}^*}\varphi^{}_t(x)|\hat{\phi}(t)|^2\hspace{0.02cm}\alpha(dt).
$$ 
Comparing to (\ref{eq:inverse_sphericaltransform}), it follows that $\hat{f}(t) = |\hat{\phi}(t)|^2$, as required. 
\hfill \proofbox
%the above proof of the only if part does not really explain why $\hat{f}$ should be positive. Here is a brief sketch of a direct proof of this fact. From (\ref{eq:generalpd}) and (\ref{eq:planch25}),
%\begin{equation} \label{eq:generalpd_proof}
%   \int_{\mathfrak{a}^*} |\hat{\psi}(t)|^2\hspace{0.02cm}(\hat{f}\alpha)(dt) \geq 0
%\end{equation}
%for any $\psi \in C^{\hspace{0.03cm}c}_H(M)$. By a density argument, (\ref{eq:generalpd_proof}) will also hold when $\psi$ is any $H$-invariant Schwartz function on $M$. However, any $W$-invariant Schwartz function $g$ on $\mathfrak{a}^*$ is of the form $g(t) = \hat{\psi}(t)$~\cite{helgason2} (Page 489). In particular, if $g$ is $W$-invariant, smooth and compactly supported,
%\begin{equation} \label{eq:alternative1}
%   \int_{\mathfrak{a}^*} |g(t)|^2\hspace{0.02cm}(\hat{f}\alpha)(dt) \geq 0
%\end{equation}
%To conclude, let $\mathfrak{a}^*_+$ be a positive Weyl chamber in $\mathfrak{a}^*$~\cite{helgason1}. Let $b$ be a smooth function with compact support conntained in $\mathfrak{a}^*_+$, and $g(t) = \sum_w b(w\cdot t)$ where the sum is over $w \in W$ (Weyl group). Since $\hat{f}\alpha$ is $W$-invariant,
%$$
%\int_{\mathfrak{a}^*} |g(t)|^2\hspace{0.02cm}(\hat{f}\alpha)(dt) = |W| \int_{\mathfrak{a}^*_+} |b(t)|^2\hspace{0.02cm}(\hat{f}\alpha)(dt)
%$$
%where $|W|$ is the order of the Weyl group. Thus, (\ref{eq:alternative1}) implies the measure $\hat{f}\alpha$ is positive against $|b|^2$ for any $b$ as above. 
%However,

%\section{Proofs for Paragraph \ref{sec:gauss_h2}} \label{sec:proofs_gauss_h2}

\subsection{Proof of Proposition \ref{prop:h2_largelambda}} \label{sec:proof_largelambda}
Let $R$ be such that $r \leq R$ implies $|\cosh(r) - 1| < 1$. Then~\cite{terras1} (Page 142)
\begin{equation} \label{eq:legendre_taylor}
  P_{-\frac{1}{2}+\mathrm{i}t}(\cosh(r)) = 1 - \left(1/4 + t^2\right)\!(r^2/4) + \varepsilon(t,r) \hspace{0.5cm} \text{for $r \leq R$}
\end{equation}
with remainder $|\varepsilon(t,r)| \leq C(T,R)r^4$ for all $t \leq T$ and $r \leq R$, where $C(T,R) > 0$ is a constant. Replacing this into (\ref{eq:h2_sphericaltransform}), the following decomposition is obtained,
\begin{equation} \label{eq:hatf_decomposition}
  \hat{f}(t) = \hat{f}_a(t) + \hat{f}_b(t) + \hat{f}_c(t)
\end{equation}
in terms of the following integrals
\begin{align}
\label{eq:hatf_decomposition_a} & \hat{f}_a(t) = \int^R_0 e^{-\lambda r^2}\left[ 1 - \left(1/4 + t^2\right)\!(r^2/4)\hspace{0.02cm}\right]\sinh(r)dr,\\[0.2cm]
\label{eq:hatf_decomposition_b} & \hat{f}_b(t) = \int^R_0 e^{-\lambda r^2}\varepsilon(t,r)\sinh(r)dr,\\[0.2cm]
\label{eq:hatf_decomposition_c} & \hat{f}_c(t) = \int^\infty_R e^{-\lambda r^2}  P_{-\frac{1}{2}+\mathrm{i}t}(\cosh(r))\sinh(r)dr.
\end{align}
Noting that $|P_{-\frac{1}{2}+\mathrm{i}t}(\cosh(r))| \leq 1$~\cite{terras1} (Page 161), a simple calculation yields
\begin{equation} \label{eq:hatfc_bound}
 |\hat{f}_c(t)| \,\leq \int^\infty_R e^{-\lambda r^2} \sinh(r)dr = O\left(e^{-\lambda R^2}\right).
\end{equation}
On the other hand, 
$$
\left(C(T,R)\right)^{-1}|\hat{f}_b(t)| \,\leq \int^R_0 e^{-\lambda r^2} r^4\hspace{0.02cm}\sinh(r)dr \,
\leq \int^\infty_0 e^{-\lambda r^2} r^4\hspace{0.02cm}\sinh(r)dr.
$$
However, by performing an elementary change of variables and then assuming that $\lambda \geq 1$, this implies the following bound,
$$
\left(C(T,R)\right)^{-1}|\hat{f}_b(t)| \leq 
\lambda^{-\frac{5}{2}}\,\int^\infty_0 e^{-u^2} u^4\hspace{0.02cm}\sinh(u) du
$$
which means that 
\begin{equation} \label{eq:hatfb_bound}
  |\hat{f}_b(t)| = O\left( \lambda^{-\frac{5}{2}}\right).
\end{equation}
Finally, $\hat{f}_a(t)$ can be extended to an integral from $0$ to $\infty$, with an exponentially small error,
\begin{equation} \label{eq:hatfa_bound_ZV}
  \hat{f}_a(t) = Z(\lambda) - \left(1/4 + t^2\right)V(\lambda) + O\left(e^{-\lambda R^2}\right)
\end{equation}
where $Z(\lambda)$ and $V(\lambda)$ are the following integrals
$$
Z(\lambda) = \int^\infty_0e^{-\lambda r^2}\sinh(r)dr,\hspace{0.5cm} V(\lambda) = \frac{1}{4}\hspace{0.03cm}\int^\infty_0 e^{-\lambda r^2} r^2\hspace{0.02cm}\sinh(r)dr.
$$
By a direct evaluation
\begin{equation} \label{eq:Zlamb}
Z(\lambda)  =  \frac{1}{2} \sqrt{\frac{\pi}{\lambda}} \exp\left(\frac{1}{4\lambda}\right) \mathrm{erf}\left(\frac{1}{2\sqrt{\lambda}}\right)  
= \frac{\lambda^{-1}}{2} + \frac{\lambda^{-2}}{6} + O\left(\lambda^{-3}\right)
\end{equation}
where $\mathrm{erf}$ denotes the error function~\cite{wongbeals} (Page 35), and then by noting that $4V(\lambda) = -Z^\prime(\lambda)$ (the prime denotes derivation),
\begin{equation} \label{eq:Vlamb}
V(\lambda) = \frac{\lambda^{-2}}{8} + O(\lambda^{-3}). 
\end{equation}
Therefore, replacing (\ref{eq:Zlamb}) and (\ref{eq:Vlamb}) into (\ref{eq:hatfa_bound_ZV}),
\begin{equation} \label{eq:hatfa_bound_O}
  \hat{f}_a(t) = \frac{\lambda^{-1}}{2} - \left(t^2 - \frac{13}{12}\right) \frac{\lambda^{-2}}{8} + O\left(\lambda^{-3}\right).
\end{equation}
To conclude the proof, replace (\ref{eq:hatfc_bound}), (\ref{eq:hatfb_bound}) and (\ref{eq:hatfa_bound_O}) into (\ref{eq:hatf_decomposition}). This yields
\begin{equation} \label{eq:hatf_final_bound}
\hat{f}(t) \sim \frac{\lambda^{-1}}{2} - \left(t^2 - \frac{13}{12}\right) \frac{\lambda^{-2}}{8} \hspace{0.5cm} \text{as $\lambda \rightarrow \infty$}
\end{equation}
($a \sim b$ means $a/b$ converges to $1$), where the convergence is uniform in $t \leq T$. It is now enough to note that the right-hand side of (\ref{eq:hatf_final_bound}) is positive for all $t \leq T$ when $\lambda \geq T^2/4$. \hfill \proofbox

\subsection{Proof of Lemma \ref{lem:h3_sine}} \label{sec:proof_h3_sine}
Note that $\sinh(r)\sin(tr)$ is an even function of $r$. Moreover, since $g(r)$ is also even, (\ref{eq:h3_sphericaltransform}) can be written
\begin{align*}
\hat{f}(t) & =  \frac{4\pi}{t}\int^\infty_0 g(r)\sinh(r)\hspace{0.02cm}\sin(tr)dr\\[0.2cm]
& = \frac{2\pi}{t}\int^{+\infty}_{-\infty} g(r)\sinh(r)\hspace{0.02cm}\sin(tr)dr
\end{align*}
In other words,
$$
\hat{f}(t) = \frac{2\pi}{t}\,\mathrm{Im}\left\lbrace \int^{+\infty}_{-\infty} g(r)\sinh(r)\hspace{0.02cm}e^{\mathrm{i}tr}dr\right\rbrace.
$$
Using the definition of the hyperbolic sine, this becomes
$$
\hat{f}(t) = \frac{\pi}{t}\,\mathrm{Im}\left\lbrace \int^{+\infty}_{-\infty} g(r)\hspace{0.02cm}e^{(\mathrm{i}t+1)r}dr - 
\int^{+\infty}_{-\infty} g(r)\hspace{0.02cm}e^{(\mathrm{i}t-1)r}dr\right\rbrace
$$
and then, by definition of the moment-generating function $G$,
$$
\hat{f}(t) = \frac{\pi}{t}\,\mathrm{Im}\left\lbrace G(\mathrm{i}t+1) -  
G(\mathrm{i}t-1)\right\rbrace
$$
which is the same as (\ref{eq:h3_sine}).
\hfill \proofbox

\subsection{Proof of Lemma \ref{lem:embedh2}} \label{sec:proof_embedh2}
It will be enough to prove that there exists a totally geodesic submanifold $M_o$ of $M$, such that $M_o$ is a hyperbolic plane. Indeed, $M$ is a Hadamard manifold (a simply connected complete Riemannian manifold with negative sectional curvatures~\cite{helgason1} (Chapter VI)). Therefore, if $M_o$ is a totally geodesic submanifold, it is also an isometrically embedded metric space. The proof will rely very heavily on the algebraic description of symmetric spaces, detailed in~\cite{helgason1} (in particular Chapters V and VI).

Recall that $M$ is determined by a symmetric pair $(G,H)$. The Lie algebra $\mathfrak{g}$ of $G$ is a real semisimple Lie algebra, and admits a Cartan involution $\theta$ with corresponding Cartan decomposition $\mathfrak{g} = \mathfrak{h} + \mathfrak{p}$ ($\mathfrak{h}$ is the Lie algebra of $H$). If $B$ is the Killing form of $\mathfrak{g}\hspace{0.03cm}$, then 
\begin{equation} \label{eq:nc_scalarprod}
\langle \xi\hspace{0.02cm},\eta\rangle = -B\!\left(\xi,\theta(\eta)\right) \hspace{0.5cm} \xi,\eta \in \mathfrak{g}
\end{equation}
defines a scalar product on $\mathfrak{g}$. 
Let $\mathfrak{a}$ be a maximal Abelian subspace of $\mathfrak{p}$. The linear maps $\mathrm{ad}_a:\mathfrak{g} \rightarrow \mathfrak{g}$ where $\mathrm{ad}_a(\xi) = [a,\xi]$ for $a \in \mathfrak{a}$ are jointly diagonalisable, with real eigenvalues. These eigenvalues are linear forms $\lambda : \mathfrak{a} \rightarrow \mathbb{R}$,  called the roots of $\mathfrak{g}$ with respect to $\mathfrak{a}$. 

Fix a non-zero root $\lambda$, and $a_\lambda \in \mathfrak{a}$ such that $\lambda(a) = B(a_\lambda\hspace{0.02cm},a)$ for $a \in \mathfrak{a}$. Then, let $\xi_\lambda$ be a joint eigenvector, $\mathrm{ad}_a(\xi_\lambda)  = \lambda(a)\hspace{0.02cm}\xi_\lambda$ for $a \in \mathfrak{a}$. Note that
\begin{equation} \label{eq:theta_xi_1}
  \left[a,\theta(\xi_\lambda)\right] = \theta\left(\left[\theta(a),\xi_\lambda\right]\right) = - \theta\left(\left[a,\xi_\lambda\right]\right)
\end{equation} 
for any $a \in \mathfrak{a}$, where the first equality follows because $\theta$ is an involution, and the second because $a \in \mathfrak{a}$, a subspace of $\mathfrak{p}$, and $\mathfrak{p}$ is by definition the $(-1)$-eigenspace of $\theta$. However,
\begin{equation} \label{eq:theta_xi_2}
\theta\left(\left[a,\xi_\lambda\right]\right) = \lambda(a)\hspace{0.02cm}\theta(\xi_\lambda)
\end{equation}
by the definition of $\xi_\lambda$ and linearity of $\theta$. Thus, (\ref{eq:theta_xi_1}) and (\ref{eq:theta_xi_2}) imply that 
$\mathrm{ad}_a(\theta(\xi_\lambda)) = -\lambda(a)\hspace{0.02cm}\theta(\xi_\lambda)$ for any $a \in \mathfrak{a}$. Denoting $\theta(\xi_\lambda)$ by $\xi_{-\lambda}$ and recalling the definition of $a_\lambda \in \mathfrak{a}$, it is now seen that the following commutation relations hold 
\begin{equation} \label{eq:commutation1}
  [a_\lambda,\xi_\lambda] = \Vert a_\lambda \Vert^2\hspace{0.02cm} \xi_\lambda \hspace{1cm}
  [a_\lambda,\xi_{-\lambda}] = -\Vert a_\lambda \Vert^2\hspace{0.02cm} \xi_{-\lambda}
\end{equation}
where $\Vert \cdot \Vert$ is the norm with respect to the scalar product (\ref{eq:nc_scalarprod}). It will now be proved that 
\begin{equation} \label{eq:commutation2}
  \left[\xi_\lambda\hspace{0.02cm},\xi_{-\lambda}\right] = -\Vert \xi_\lambda\Vert^2\hspace{0.02cm}a_\lambda.
\end{equation}
First, any $a \in \mathfrak{a}$ commutes with $\left[\xi_\lambda\hspace{0.02cm},\xi_{-\lambda}\right]$. Indeed, 
\begin{align*}
\left[a\hspace{0.02cm},\left[\xi_{\lambda}\hspace{0.02cm},\xi_{-\lambda}\right]\right] &= 
\left[\left[a\hspace{0.02cm},\xi_{\lambda}\right]\!,\xi_{-\lambda}\right] + \left[\xi_{\lambda}\hspace{0.02cm},\left[a\hspace{0.02cm},\xi_{-\lambda}\right]\right] \\[0.2cm]
&= \lambda(a)[\xi_\lambda\hspace{0.02cm},\xi_{-\lambda}] - \lambda(a)[\xi_\lambda\hspace{0.02cm},\xi_{-\lambda}]
\end{align*}
where the first equality follows from Jacobi's identity for Lie brackets, and the second because $\xi_\lambda$ and $\xi_{-\lambda}$ are joint eigenvectors with eigenvalues $\lambda(a)$ and $-\lambda(a)$. Second, $\left[\xi_\lambda\hspace{0.02cm},\xi_{-\lambda}\right] \in \mathfrak{p}$. Indeed, $\theta(\left[\xi_\lambda\hspace{0.02cm},\xi_{-\lambda}\right]) = \left[\theta(\xi_\lambda)\hspace{0.02cm},\theta(\xi_{-\lambda})\right]$, but $\theta(\xi_\lambda) = \xi_{-\lambda}$ by definition, and $\theta(\xi_{-\lambda}) = \xi_\lambda$ since $\theta$ is an involution. Therefore, 
$\theta(\left[\xi_\lambda\hspace{0.02cm},\xi_{-\lambda}\right]) = 
-\left[\xi_\lambda\hspace{0.02cm},\xi_{-\lambda}\right]$ which means $\left[\xi_\lambda\hspace{0.02cm},\xi_{-\lambda}\right] \in \mathfrak{p}$. However, $\mathfrak{a}$ is maximal Abelian in $\mathfrak{p}$, and this yields $\left[\xi_\lambda\hspace{0.02cm},\xi_{-\lambda}\right] \in \mathfrak{a}$. 

Finally, to obtain (\ref{eq:commutation2}), note that
$$
B\!\left(\left[\xi_\lambda\hspace{0.02cm},\xi_{-\lambda}\right]\!,a\right) = 
B\!\left(\xi_{-\lambda}\hspace{0.02cm},\left[a\hspace{0.02cm},\xi_\lambda\right]\right) = B(\xi_{-\lambda}\hspace{0.02cm},\xi_\lambda)\hspace{0.02cm}\lambda(a)
$$
where the first equality holds since $B$ is $\mathrm{ad}$-invariant, and the second since $\xi_\lambda$ is a joint eigenvector with eigenvalue $\lambda(a)$. Now, the fact that $B$ is symmetric and the definition of $a_\lambda$ imply $B\!\left(\left[\xi_\lambda\hspace{0.02cm},\xi_{-\lambda}\right]\!,a\right) = B\!\left(\xi_{\lambda}\hspace{0.02cm},\xi_{-\lambda}\right)B\!\left(a_\lambda\hspace{0.02cm},a\right)$, so that
$$
\left[\xi_\lambda\hspace{0.02cm},\xi_{-\lambda}\right] = B\!\left(\xi_{\lambda}\hspace{0.02cm},\xi_{-\lambda}\right)a_\lambda
$$
and (\ref{eq:commutation2}) follows after recalling (\ref{eq:nc_scalarprod}) and $\xi_{-\lambda} = \theta(\xi_\lambda)$.

Moving on, let $\mathfrak{g}_{\hspace{0.02cm} o}$ denote the Lie subalgebra of $\mathfrak{g}$, generated by $\lbrace \xi_{-\lambda}, a_\lambda,\xi_\lambda\rbrace$. The Lie bracket of $\mathfrak{g}_{\hspace{0.02cm} o}$ is completely determined by (\ref{eq:commutation1}) and (\ref{eq:commutation2}). It is straightforward to see that $\mathfrak{g}_{\hspace{0.02cm} o} \simeq \mathfrak{sl}(2,\mathbb{R})$ (the Lie algebra of $2 \times 2$ zero-trace real matrices). The isomorphism is given by
\begin{equation} \label{eq:sl2r_isomorphism}
\xi_{-\lambda} \longleftrightarrow \left(\begin{array}{cc}  0&0 \\[0.1cm] -1 &0 \end{array}\right) \hspace{0.2cm};\hspace{0.2cm}
a_{\lambda} \longleftrightarrow \left(\begin{array}{cc} 1 & 0\\[0.1cm]  0& -1\end{array}\right) \hspace{0.2cm};\hspace{0.2cm}
\xi_\lambda \longleftrightarrow \left(\begin{array}{cc}0  & 1 \\[0.1cm] 0 & 0\end{array}\right) 
\end{equation}
after proper rescaling of $\lbrace \xi_{-\lambda},a_\lambda,x_{-\lambda}\rbrace$ (this needs to ensure that $\Vert a_\lambda \Vert^2 = 2$ and $\Vert\xi_\lambda\Vert^2 = 1$). Let $G_o$ be the connected subgroup of $G$ with Lie algebra $\mathfrak{g}_{\hspace{0.02cm} o}$ and let $M_o$ be the orbit of $o$ under $G_o$ (recall that $o \in M$ has stabiliser $H$ in $G$ and can be thought of as the ``\,point of origin\,"). 
Then, $M_o$ is a totally geodesic submanifold of $M$, because $\mathfrak{g}_{\hspace{0.02cm} o}$ is a Lie subalgebra of $\mathfrak{g}$~\cite{helgason1} (Page 189). Moreover, $M_o$ is a Riemannian homogeneous space under the action of $G_o\hspace{0.02cm}$, $M_o \simeq G_o/H_o$ where $H_o = G_o \cap H$ is the stabiliser of $o$ in $G_o\hspace{0.02cm}$. To see that $M_o$ is a hyperbolic plane (up to an irrelevant scaling factor), note first that $G_o \simeq SL(2,\mathbb{R})$. It then remains to check tha $H_o \simeq SO(2)$, since the Riemannian homogeneous space $SL(2,\mathbb{R})/SO(2)$ is just a hyperbolic plane~\cite{terras1}. To do so, let $\mathfrak{h}_o$ be the Lie algebra of $H_o\hspace{0.03cm}$. Then, $\mathfrak{h}_o = \mathfrak{g}_{\hspace{0.02cm} o} \cap \mathfrak{h}$. However, since $\mathfrak{h}$ is the $(+1)$-eigenspace of $\theta$, it is easy to see that $\mathfrak{h}_o$ is generated by $\tau = \xi_\lambda + \xi_{-\lambda}\hspace{0.02cm}$. Since $\tau$ corresponds to the generator of $SO(2) \subset SL(2,\mathbb{R})$, under the isomorphism (\ref{eq:sl2r_isomorphism}),  
$H_o \simeq SO(2)$, as required. 
\hfill \proofbox

\section*{Acknowledgments}

The authors acknowledge financial support from the School of Physical and Mathematical Sciences, the Talent Recruitment and Career Support (TRACS) Office, and the Presidential Postdoctoral Programme at Nanyang Technological University (NTU).

\bibliographystyle{siamplain}
\bibliography{references}

\newpage

\appendix

\section{L$^\textbf{\!1}$-Godement theorem} \label{sec:l1godement}
The L$^{\!\scriptscriptstyle 2}$-Godement theorem, stated back in Section \ref{sec:l2godement}, has the following variant, which will be called the L$^{\!\scriptscriptstyle 1}$-$\hspace{0.02cm}$Godement theorem. It seems useful to have this at hand, to deal with kernels $k$ such that the function $f(x) = k(o,x)$\hfill\linebreak is integrable but not square-integrable. 
\begin{theorem}[L$^\textbf{\!1}$-$\hspace{0.02cm}$Godement Theorem]
  A continuous function $f \in L^1_H(M)$ is positive-definite if and only if $\hat{f}(t) \geq 0$ for all $t \in \mathfrak{a}^*$ and $\hat{f}$ is integrable with respect to the measure $|c(t)|^{-2}dt$.
\end{theorem}
Here, $L^1_H(M)$ denotes the space of $H$-invariant integrable functions (with respect to $\mathrm{vol}$). The following proof is modeled on the one in Section \ref{sec:prop_l2god}. In particular, the notation is the same as in Paragraph \ref{subsec:notation}.
\subsection{The if part} It will be enough to prove that the inversion formula (\ref{eq:inverse_sphericaltransform}) holds
\begin{equation} \label{eq:inversion_l1}
f(x) = \int_{\mathfrak{a}^*} \varphi^{}_t(x)\hat{f}(t)\hspace{0.02cm}\alpha(dt) \hspace{0.5cm} \text{for all $x \in M$}.
\end{equation}
The fact that $f$ is continuous and positive-definite then follows exactly as in Paragraph \ref{subsec:ifpart}. 
Denote by $G_s(x) = G_s(o,x)$ the heat kernel of $M$. For $s > 0$, $G_s$ has spherical transform~\cite{gangolli}
\begin{equation} \label{eq:heat}
  \hat{G}_s(t) = \left(Z(s)\right)^{-1}\exp\left[-s\left(\Vert t\Vert^2 + \Vert \rho \Vert^2\right) \right]
\end{equation}
where $Z(s)$ is a normalizing constant, which guarantees that $G_s$ integrates to $1$ over $M$, and where $\rho$ is half the sum of positive roots (as in Formula (\ref{eq:spherical_nc})). Recall that the heat semigroup is strongly continuous on $L^1_H(M)$~\cite{davies} (Page 148). This means the convolutions $f_s = G_s * f$ converge to $f$ in $L^1_H(M)$ as $s \rightarrow 0$. 
Of course, $G_s$ satisfies the inversion formula (\ref{eq:inverse_sphericaltransform}), 
$$
G_s(x) = \int_{\mathfrak{a}^*} \varphi_t(x)\hspace{0.03cm}\hat{G}_s(t)\hspace{0.02cm}\alpha(dt).
$$
Therefore, by taking the convolution under the integral
\begin{equation} \label{eq:preinversion_l1}
 f_s(x) = \int_{\mathfrak{a}^*} (f*\varphi^{}_t)(x)\hspace{0.03cm}\hat{G}_s(t)\hspace{0.02cm}\alpha(dt) = 
\int_{\mathfrak{a}^*} \varphi^{}_t(x)\hat{f}(t)\hspace{0.03cm}\hat{G}_s(t)\hspace{0.02cm}\alpha(dt)
\end{equation}
where the second equality follows because $\varphi^{}_t$ is an eigenfunction of convolution operators~\cite{terras2} (Page 67).
Now, from the assumption that $\hat{f}$ is integrable (\textit{i.e.} $\hat{f} \in L^1_W(\mathfrak{a}^*)$), and from (\ref{eq:heat}), the dominated convergence theorem can be applied in (\ref{eq:preinversion_l1}), in order to obtain
$$
\lim f_s(x) = \int_{\mathfrak{a}^*}\varphi^{}_t(x)\hat{f}(t)\hspace{0.02cm}\alpha(dt) 
$$
as $s \rightarrow 0$. But this limit is identical to $f(x)$, so (\ref{eq:inversion_l1}) is proved. 
\subsection{The only-if part} As in Paragraph \ref{subsec:onlyif}, the aim is to prove (\ref{eq:planch25}), 
\begin{equation} \label{eq:planch25_bis}
\int_M (\psi*\psi^{\scriptscriptstyle \dagger})(x)f(x)\hspace{0.02cm}\mathrm{vol}(dx) = 
\int_{\mathfrak{a}^*} |\hat{\psi}(t)|^2\hat{f}(t)\hspace{0.02cm}\alpha(dt)
\end{equation}
which then implies the spectral representation (\ref{eq:planch2}). Since $\psi*\psi^{\scriptscriptstyle \dagger} \in C^{\hspace{0.03cm}c}_H(M)$, it follows from (\ref{eq:inverse_sphericaltransform}) and (\ref{eq:convtheorem}) that
$$
(\psi*\psi^{\scriptscriptstyle \dagger})(x) = \int_{\mathfrak{a}^*} \varphi^{}_t(x)\hspace{0.02cm}|\hat{\psi}(t)|^2\hspace{0.02cm}\alpha(dt).
$$
Replacing this under the integral on the left-hand side of (\ref{eq:planch25_bis}), the right-hand side is obtained by using the definition (\ref{eq:sphericaltransform}) of $\hat{f}$. Now, everything follows from (\ref{eq:planch2}) as in Paragraph \ref{subsec:onlyif}. 
\hfill \proofbox

\section{The hyperbolic secant in three dimensions} \label{app: hyperbolic secant 3D}
Recall the hyperbolic secant kernel from \ref{subsubsec:hsecant}. This is here considered in the case where $M$ is a hyperbolic space. The aim is to apply the 
L$^{\!\scriptscriptstyle 2}$-$\hspace{0.02cm}$Godement theorem, with the help of the spherical transform pair (\ref{eq:h3_sphericaltransform}). 

Setting $g(r) = (\cosh(r))^{-a}$ and $f(x) = g(d(o,x))$, it follows that $f$ is continuous and in $L^2_H(M)$ when $a > 1$. This is because,
$$
\int_M|f(x)|^2 = 4\pi\int^\infty_0 (\cosh(r))^{-2a}\hspace{0.02cm}\sinh^2(r)dr = \int^\infty_1 u^{-2a}(u^2-1)^{\frac{1}{2}}du \,< \infty
$$
as follows by integrating in geodesic spherical coordinates. Thus, the L$^{\!\scriptscriptstyle 2}$-$\hspace{0.02cm}$Godement theorem can be applied to $f$. The spherical transform $\hat{f}$ can be found from Lemma \ref{lem:h3_sine}, at least for integer values of $a$. Indeed, the moment-generating function $G$, required in the lemma, reads
$$
G(s) = \int^{+\infty}_{-\infty} (\cosh(r))^{-a}\,e^{sr}dr.
$$
When $a$ is an (of course, positive) integer, this can be obtained using the residue theorem~\cite{secantlaw}. Precisely, if $a$ is odd, $a = 2n+1$,
\begin{equation} \label{eq:secantMGF_odd}
  G(\mathrm{i}\sigma) = \frac{\pi\,\mathrm{sech}(\pi \sigma/2)}{(2n)!}\prod^{n-1}_{m=0} \left[\sigma^2 + (2m+1)^2\hspace{0.02cm}\right]
\end{equation}
and if $a$ is even, $a = 2n$,  
\begin{equation} \label{eq:secantMGF_even}
  G(\mathrm{i}\sigma) = \frac{\pi \sigma\,\mathrm{csch}(\pi \sigma/2)}{(2n-1)!}\prod^{n-1}_{m=1} \left[\sigma^2 + (2m)^2\hspace{0.02cm}\right]
\end{equation}
where $\mathrm{sech} = 1/\cosh$ and $\mathrm{csch} = 1/\sinh$ are the hyperbolic secant and cosecant, and where the imaginary part of $\sigma$ is $\geq -1$. With $G$ given by (\ref{eq:secantMGF_odd}) and (\ref{eq:secantMGF_even}), the spherical transform $\hat{f}$ can be computed using (\ref{eq:h3_sine}). Letting $H(\sigma) = G(\mathrm{i}\sigma)$, this yields
\begin{equation} \label{eq:h3_spherical_sech}
\hat{f}(t) =   \frac{\pi}{t}\,\mathrm{Im}\left\lbrace
H(t-\mathrm{i}) - H(t+\mathrm{i})\right\rbrace.
\end{equation}

\begin{lemma} \label{app:Gamma id}
    If $z\in \mathbb{C}\setminus \mathbb{Z}$ and $n$ is a positive integer, then
    \begin{align*}
        \Gamma(n+z)\Gamma(n-z) &= -\frac{\pi}{z\sin(\pi z)}\prod_{m=0}^{n-1}\left(m^2-z^2\right), \\
        \Gamma\left(n+\frac{1}{2}+z\right)\Gamma\left(n+\frac{1}{2}-z\right) &= \frac{\pi}{\cos(\pi z)}\prod_{m=0}^{n-1}\left[\left(m+\frac{1}{2}\right)^2-z^2\right].
    \end{align*}
\end{lemma}

\begin{proof}
The proof of the first equality follows by induction using Euler's reflection formula $\Gamma(z)\Gamma(1-z)=\pi/\sin(\pi z)$ for $z\notin \mathbb{Z}$ and the identity $\Gamma(z+1)=z\Gamma(z)$. The second equality follows in the same manner using the formula $\Gamma(1/2+z)\Gamma(1/2-z)=\pi/\cos(\pi z)$, which itself follows from replacing $z$ with $z+1/2$ in Euler's reflection formula. 
\end{proof}

\begin{proposition}
The hyperbolic secant kernel is positive-definite in three-dimensional hyperbolic space whenever $a$ is a positive integer.
\end{proposition}

\begin{proof}
If $a$ is odd with $a = 2n+1$, then
\begin{align*}
    H(\sigma)=G(i\sigma)
&=\frac{4^n\pi\,\mathrm{sech}(\pi \sigma/2)}{(2n)!}\prod^{n-1}_{m=0} \left[ \left(m+\frac{1}{2}\right)^2-\left(\frac{i\sigma}{2}\right)^2\hspace{0.02cm}\right] \\
&= \frac{4^n}{(2n)!}\hspace{0.02cm}\Gamma\left(n+\frac{1}{2}+i\frac{\sigma}{2}\right)\Gamma\left(n+\frac{1}{2}-i\frac{\sigma}{2}\right)
\end{align*}
using Lemma \ref{app:Gamma id} with $z=i\sigma/2$. Thus,
\begin{align*}
H(t-\mathrm{i}) &- H(t+\mathrm{i})
= \frac{4^n}{(2n)!}\left[
\Gamma\left(n+1+\frac{it}{2}\right)\Gamma\left(n-\frac{it}{2}\right)-\Gamma\left(n+\frac{it}{2}\right)\Gamma\left(n+1-\frac{it}{2}\right)
\right] \\
&=\frac{4^n}{(2n)!}\left[\left(n+\frac{it}{2}\right)\Gamma\left(n+\frac{it}{2}\right)\Gamma\left(n-\frac{it}{2}\right)-\left(n-\frac{it}{2}\right)\Gamma\left(n+\frac{it}{2}\right)\Gamma\left(n-\frac{it}{2}\right)\right] \\
&= i\frac{4^n t}{(2n)!}\hspace{0.02cm}\Gamma\left(n+\frac{it}{2}\right)\Gamma\left(n-\frac{it}{2}\right) = i\frac{4^n t}{(2n)!}\hspace{0.02cm}\left|\Gamma\left(n+\frac{it}{2}\right)\right|^2
\end{align*}
where we have used the fact that $t \in \mathbb{R}$ and $\overline{\Gamma(z)}=\Gamma(\overline{z})$ for $z\in \mathbb{C}$. Thus, (\ref{eq:h3_spherical_sech}) reduces to
\begin{equation*}
   \hat{f}(t) =  \frac{4^n \pi}{(2n)!}\hspace{0.02cm}\left|\Gamma\left(n+\frac{it}{2}\right)\right|^2 \geq 0
\end{equation*}
for any $t\in \mathbb{R}$. If $a$ is even with $a = 2n$, a similar calculation using Lemma \ref{app:Gamma id} yields
\begin{equation*}
   \hat{f}(t) =  \frac{4^n n \pi}{(2n)!}\hspace{0.02cm}\left|\Gamma\left(n-\frac{1}{2}+\frac{it}{2}\right)\right|^2 \geq 0
\end{equation*}
for any $t\in\mathbb{R}$, which completes the proof.
\end{proof}

\end{document}